\documentclass[sigconf]{aamas} 

\pdfoutput=1

\usepackage{balance} % for balancing columns on the final page

\setcopyright{ifaamas}
\acmConference[AAMAS '24]{Proc.\@ of the 23rd International Conference
on Autonomous Agents and Multiagent Systems (AAMAS 2024)}{May 6 -- 10, 2024}
{Auckland, New Zealand}{N.~Alechina, V.~Dignum, M.~Dastani, J.S.~Sichman (eds.)}
\copyrightyear{2024}
\acmYear{2024}
\acmDOI{}
\acmPrice{}
\acmISBN{}

\acmSubmissionID{1037}

\title[AAMAS-2024 Formatting Instructions]{Disentangling Policy from Offline Task Representation Learning via Adversarial Data Augmentation}

\author{Chengxing Jia$^*$}
\affiliation{
  \institution{National Key Laboratory for Novel Software Technology}
  \institution{School of Artificial Intelligence}
  \city{Nanjing University, Nanjing}
  \country{China}
  \city{Polixir Technologies, Nanjing}
  \country{China}
}
\email{jiacx@lamda.nju.edu.cn}

\author{Fuxiang Zhang$ ^*$}
\affiliation{
  \institution{National Key Laboratory for Novel Software Technology}
  \institution{School of Artificial Intelligence}
  \city{Nanjing University, Nanjing}
  \country{China}
}
\email{zhangfx@lamda.nju.edu.cn}

\author{Yi-Chen Li}
\affiliation{
  \institution{National Key Laboratory for Novel Software Technology}
  \institution{School of Artificial Intelligence}
  \city{Nanjing University, Nanjing}
  \country{China}
  \city{Polixir Technologies, Nanjing}
  \country{China}
}
\email{liyc@lamda.nju.edu.cn}

\author{Chen-Xiao Gao}
\affiliation{
  \institution{National Key Laboratory for Novel Software Technology}
  \institution{School of Artificial Intelligence}
  \city{Nanjing University, Nanjing}
  \country{China}
}
\email{gaocx@lamda.nju.edu.cn}

\author{Xu-Hui Liu}
\affiliation{
  \institution{National Key Laboratory for Novel Software Technology}
  \institution{School of Artificial Intelligence}
  \city{Nanjing University, Nanjing}
  \country{China}
  \city{Polixir Technologies, Nanjing}
  \country{China}
}
\email{liuxh@lamda.nju.edu.cn}

\author{Lei Yuan}
\affiliation{
  \institution{National Key Laboratory for Novel Software Technology}
  \institution{School of Artificial Intelligence}
  \city{Nanjing University, Nanjing}
  \country{China}
  \city{Polixir Technologies, Nanjing}
  \country{China}
}
\email{yuanl@lamda.nju.edu.cn}

\author{Zongzhang Zhang}
\affiliation{
  \institution{National Key Laboratory for Novel Software Technology}
  \institution{School of Artificial Intelligence}
  \city{Nanjing University, Nanjing}
  \country{China}
}
\email{zzzhang@nju.edu.cn}

\author{Yang Yu$^\dagger$}
\affiliation{
  \institution{National Key Laboratory for Novel Software Technology}
  \institution{School of Artificial Intelligence}
  \city{Nanjing University, Nanjing}
  \country{China}
  \city{Polixir Technologies, Nanjing}
  \country{China}
}
\email{yuy@nju.edu.cn}

\begin{abstract}
\footnotetext{These authors contributed equally.}
\footnotetext{Corresponding author.}

Offline meta-reinforcement learning (OMRL) proficiently allows an agent to tackle novel tasks while solely relying on a static dataset. For precise and efficient task identification, existing OMRL research suggests learning separate task representations that be incorporated with policy input, thus forming a context-based meta-policy. A major approach to train task representations is to adopt contrastive learning using multi-task offline data. The dataset typically encompasses interactions from various policies (i.e., the behavior policies), thus providing a plethora of contextual information regarding different tasks. Nonetheless, amassing data from a substantial number of policies is not only impractical but also often unattainable in realistic settings. Instead, we resort to a more constrained yet practical scenario, where multi-task data collection occurs with a limited number of policies. We observed that learned task representations from previous OMRL methods tend to correlate spuriously with the behavior policy instead of reflecting the essential characteristics of the task, resulting in unfavorable out-of-distribution generalization. To alleviate this issue, we introduce a novel algorithm to disentangle the impact of behavior policy from task representation learning through a process called adversarial data augmentation. Specifically, the objective of adversarial data augmentation is not merely to generate data analogous to offline data distribution; instead, it aims to create adversarial examples designed to confound learned task representations and lead to incorrect task identification. Our experiments show that learning from such adversarial samples significantly enhances the robustness and effectiveness of the task identification process and realizes satisfactory out-of-distribution generalization. The results in MuJoCo locomotion tasks demonstrate that our approach surpasses other OMRL baselines across various meta-learning task sets.
\end{abstract}

\keywords{Offline reinforcement learning; Meta-reinforcement learning}

\usepackage{xr}
\usepackage{bm}
\usepackage{makecell}
\usepackage{multirow}
\usepackage{algorithm}
\usepackage{algorithmic}
\usepackage{wrapfig}
\usepackage{lipsum,booktabs}
\usepackage{footnote}

\theoremstyle{plain}
\newtheorem{theorem}{Theorem}[section]

\theoremstyle{definition}
\newtheorem{definition}[theorem]{Definition}
\newtheorem{assumption}[theorem]{Assumption}
\theoremstyle{remark}

\usepackage{xcolor}

\newcommand{\BibTeX}{\rm B\kern-.05em{\sc i\kern-.025em b}\kern-.08em\TeX}

\makeatletter
\gdef\@copyrightpermission{
	\begin{minipage}{0.3\columnwidth}
		\href{https://creativecommons.org/licenses/by/4.0/}{\includegraphics[width=0.90\textwidth]{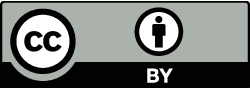}}
	\end{minipage}\hfill
	\begin{minipage}{0.7\columnwidth}
		\href{https://creativecommons.org/licenses/by/4.0/}{This work is licensed under a Creative Commons Attribution International 4.0 License.}
	\end{minipage}
	\vspace{5pt}
}
\makeatother

\begin{document}

%%% The following commands remove the headers in your paper. For final 
%%% papers, these will be inserted during the pagination process.

\pagestyle{fancy}
\fancyhead{}

%%% The next command prints the information defined in the preamble.

\maketitle 

%%%%%%%%%%%%%%%%%%%%%%%%%%%%%%%%%%%%%%%%%%%%%%%%%%%%%%%%%%%%%%%%%%%%%%%%

\section{Introduction}

In recent years, reinforcement learning (RL) has seen substantial progress, primarily in situations where abundant interactive data is accessible for the learning of policies~\citep{alphastar, alphazero, rlcyclegan}. However, the practical implementation of RL in fields such as robotics~\citep{DBLP:conf/icra/JohanninkBNLKLO19} and healthcare~\citep{gottesman2019guidelines} remains constrained by persistent challenges relating to data efficiency and generalization~\citep{introrl, generalization_survey}, especially when environmental interaction becomes costly or poses potential risks. Offline RL~\citep{levine_offlinerl_survey, bcq, brac, bear, cql, iql, prdc}, proposing to learn policy contexts from a pre-existing environmental dataset, offers a viable solution to address the safety concerns and sample efficiency of RL. In parallel, meta-RL~\citep{maml, rl2, varibad, pearl} aims to train generalizable policies, adaptable to various tasks sharing common structural aspects in terms of environmental dynamics and rewards.

As the intersection of offline RL and meta-RL, Offline Meta-RL (OMRL)~\citep{omrl_survey, metadiff, taskinf, omoff2on} garners growing recognition due to its potential to enhance both data efficiency and generalization. OMRL usually relies on a multi-task offline dataset collected by different behavior policies to train a meta-policy that can effectively adapt to various tasks within the range. 
Prior OMRL works \citep{mbml, smac, focal} focus on context-based methods using a separate context encoder to extract task representations from context data. The context encoder learning is crucial for task identification and meta task generalization.
% Recent studies integrate contrastive learning to optimize task representation learning, such as FOCAL \citep{focal}, proposing a more stable contrastive learning objective, and CORRO \citep{corro}, adopting a bi-level structure to enhance robustness.
% Recent research incorporates contrastive learning, such as FOCAL \citep{focal}, which utilizes a novel objective to balance the data distribution during task feature learning and CORRO \citep{corro}, which adopts a bi-level structure to ensure robust feature extraction, to enhance identification performance and ensure robustness.
Although current OMRL methods show tremendous potential on generalization to unseen tasks without the requirement of real-time interactions, they usually posit a comprehensive coverage of the offline dataset~\citep{focal, corro}, i.e., data from each task is collected by diverse behavior policies. Consequently, the context encoder, employed for task identity identification, can naturally develop a robust manner from diverse data. Nevertheless, in instances where interactions entail substantial risk or cost, such a presumption or condition lacks realism. We, therefore, contemplate a more feasible scenario where the data emerges from a limited number of policies, resulting in poor data coverage for task representation extraction. When deployed in tasks, the meta-policy may encounter contexts that deviate from the trained dataset, resulting in out-of-distribution scenarios, and subsequently, a failure to generalize effectively for the given tasks.

% As the intersection of offline RL and meta-RL, Offline Meta-RL (OMRL)~\citep{omrl_survey, mbml} draws increasing attention for its potential of improving both data efficiency and generalization. OMRL usually relies on a multi-task offline dataset collected by different behavior policies to train a meta-policy that can effectively adapt to various tasks within the range. This work focuses on context-based OMRL methods~\citep{mbml, smac, focal}, which train a separate context encoder to extract task representations from context data like transition pairs and sub-trajectories. Although current OMRL methods show tremendous potential on generalization to unseen tasks without the requirement of real-time interactions, they usually posit a comprehensive coverage of the offline dataset~\citep{focal, corro}, i.e., data from each task is collected by diverse behavior policies. Therefore, the context encoder, which is used to recognize task identities, can naturally develop a robust manner from diverse data. However, in scenarios where interactions are costly and risky, such a hypothesis or condition is unrealistic. We thus consider a more practical setting where data is generated by a limited amount of policies, which leads to a poor coverage of dataset for task feature extraction. 

Recent studies integrate contrastive learning to optimize task representation learning, such as FOCAL \citep{focal}, proposing a more stable contrastive learning objective, and CORRO \citep{corro}, adopting a bi-level structure to enhance robustness. When conducting experiments in such a poor-coverage dataset, we find that task representations through naive contrastive learning correlate spuriously with behavior policies from the dataset rather than recognizing task characteristics. The reason comes from the fact that the state-action sequence, which composes the context data, is formed not only by the environment nature but also by the decision from behavior policies. In this case, the context encoder is prone to fit the properties of behavior policies to output task representations. Such a manner severely degrades policy performance as the context encoder may treat data collected by different policies as different tasks and vice versa. 
While existing methods develop various contrastive learning strategies to comprehend task representations~\citep{focal, corro} that better outline task properties, the spurious relationship still remains since these methods do not account for the effect of behavior policies. 
Therefore, employing a task representation learning methodology that can disentangle behavior policies from learned representations would significantly enhance out-of-distribution generalization, subsequently facilitating adaptation to a spectrum of unseen tasks.

To disentangle the effect of behavior policy from task representations, we hope the context encoder to solely capture task-relevant data from the given context. Specifically, we find that this spurious relationship is amplified primarily when the task variation involves different environment transitions rather than reward functions. Since rewards are independently computed and will not affect subsequent trajectories, prior research can effectively address it by sharing data and relabeling rewards across tasks~\citep{mbml, cds}. Conversely, the variation of transition functions will result in different context data coupled with behavior policies, which cannot be simply relabeled. Consequently, our focus shifts to establishing a mapping process from the transition information into acquired task representations.

To eliminate the impact of behavior policies in original offline data, we introduce adversarial data augmentation to assist task representation learning. Different from traditional data augmentation approaches that usually generate in-distribution data, the proposed adversarial data augmentation is intended to reshape data distribution from the original offline dataset so that the bias induced by behavior policies can be removed. To be more specific, the data augmentation process tries to generate the most indistinguishable interaction data to confound the context encoder within an adversarial learning objective. We adopt a model-based RL approach to generate such data without environment interactions via multiple pretrained dynamics models and a particular adversarial policy for data collection. Using augmented training data, the context encoder learns to marginalize the effect of behavior policies and correctly identify tasks based on environmental characteristics. Our results on several meta-RL task sets from locomotion benchmarks demonstrate that our OMRL method with adversarial data augmentation exhibits definite abilities to distinguish tasks and significantly outperforms other OMRL baselines even when encountering out-of-distribution context data. 

%%%%%%%%%%%%%%%%%%%%%%%%%%%%%%%%%%%%%%%%%%%%%%%%%%%%%%%%%%%%%%%%%%%%%%%%

\section{Preliminaries}
% This section will briefly introduce the background on offline meta-RL and task representation learning.

\subsection{Offline Meta-RL}

Reinforcement learning (RL) tasks can be formulated as Markov decision processes (MDPs)~\cite{sutton2018reinforcement} $M = (\mathcal{S}, \mathcal{A}, T, r, d_0, \gamma)$, where $\mathcal{S}$ is the state space, $\mathcal{A}$ is the action space, $T(s'|s,a)$ is the transition function, $r(s, a)$ is the reward function, $d_0(s)$ is the initial state distribution, and $\gamma \in [0,1)$ is the discount factor. In state $s$, an agent takes action $a\sim \pi(\cdot\mid s)$ according to its policy $\pi(a\mid s)$ and results in the next state $s'\sim T(\cdot \mid s, a)$ and the reward $r(s, a)$. The objective of RL is to maximize the cumulative discounted reward (a.k.a. return) as follows:
\begin{equation*}
\max_{\pi} R_M(\pi) = \mathbb{E}_{s_0 \sim d_0, a_t \sim \pi(\cdot \mid s_t), s_{t+1}\sim T(\cdot \mid s_t, a_t)}\sum\nolimits_{t=0}^{\infty}\gamma^t r(s_t, a_t). 
\end{equation*}

Meta-RL considers deploying a meta-policy that can perform well over a task distribution $P(M)$, where the meta-policy is trained in a task set and is deployed for tasks that may be unseen during the meta-train period. Offline meta-RL (OMRL) then extends this concept to offline RL settings, where the meta-policy should be learned from a static dataset $\{D_i\}_{i=1}^n$ containing interactions from different tasks $M_1, M_2, \ldots, M_n$. Every single-task data $D_i$ is collected by some unknown behavior policies, which may be trained with modern single-task RL algorithms, for example, TD3~\citep{td3} and SAC~\cite{sac} for MuJoCo locomotion tasks~\citep{mujoco}. The goal of OMRL is to learn a meta-policy $\pi$ that maximizes the cumulative discounted reward of each task under the task distribution $P(M)$ as follows: 
\begin{equation*}
    \max_{\pi} \mathbb{E}_{M\sim P(\cdot)} \left[ R_M(\pi) \right].
\end{equation*}

It is worth noting that OMRL does not posit the number of behavior policies for single-task data. Traditionally, previous approaches~\cite{focal, corro} adopt a variety of policy checkpoints to collect interaction data for each task, resulting in a diverse data distribution for task identification. Such comprehensive data coverage may be plausible to alleviate the data bias induced by specific behavior policies. However, using exhaustive policies to collect data is not practical for realistic scenarios, which, in most cases, is what OMRL methods should be applied to. 

\subsection{Task Representation Learning}
\label{subsec:repr-learning}

Task identification is vital for meta-RL to realize fast adaptation to unseen tasks~\citep{metarl-survey}. Here we focus on context-based meta-RL methods, which learn a meta-policy $\pi(a\mid s, z)$ that is explicitly conditioned on a task representation $z$. It additionally adopts a context encoder $\phi(z\mid x)$ to extract task representations $z$ from context data $x$. The context data $x$, which is often a transition tuple or a sub-trajectory of multiple interaction steps, can be collected from an arbitrary policy. This feature enables zero-shot generalization of the meta-policy when providing pre-collected or bootstrapped context data without extra updates during the adaptation phase. 

% In order to learn a meta-policy capable of adapting to various tasks, OMRL methods generally adopt a context-based meta-learning approach to extract effective task representations from the interaction data. This process is formalized by a context encoder $\phi(z|x)$, which takes in the context $x$ related to a specific task and outputs the task representation $z$. The context-based meta-policy, denoted as $\pi(a|s, z)$, requires $z$ as an additional input alongside the state $s$ to make appropriate decisions.

Although online meta-RL often trains the context encoder and the meta-policy jointly~\citep{pearl}, in OMRL settings they are typically learned separately with different objectives to stabilize the training process. From the perspective of information theory, an ideal context encoder should maximize the mutual information $I(z;M)$ between the task representation $z$ and the corresponding MDP $M$. As $I(z;M)$ is not directly tractable, we proceed by optimizing an InfoNCE objective~\citep{infonce}, which is a lower bound of the mutual information metric as proved by~\citep{corro}:
\begin{equation}
    \label{eq:infonce}
    \begin{aligned}
        I(z;M) \ge \mathbb{E}_{M \sim P(\cdot), z, z^* \sim \phi(\cdot \mid x_M)} \log \frac{\exp (S(z, z^*))}{\Phi(z)} + C, \\
        \text{where } \Phi(z) = \sum\nolimits_{M' \sim P(\cdot), z' \sim \phi(\cdot \mid x_{M'}) } \exp(S(z,z')).
    \end{aligned}
\end{equation}
We denote $x_M$ and $x_{M'}$ as context data from tasks $M$ and $M'$ respectively, and $S(\cdot, \cdot)$ as a score function which measures the similarity between two representations, e.g., their inner product or cosine similarity. $C$ is a constant that relates to the task distribution $P(M)$. 

An intuitive understanding of maximizing the lower bound in Equation~\eqref{eq:infonce} is to make representations of the same task similar while keeping those of different tasks apart. Nevertheless, naively applying this objective may fail to establish the correct relationship between the characteristics and representation of the tasks in the offline setting, because $x_M$ is sampled from the offline dataset and influenced by the behavior policy and task $M$ jointly. When each task data $D_i$ is collected by specific behavior policies and comparatively limited in coverage, the context encoder is prone to overfitting the behavior policy but fails to generalize to unseen data distributions during evaluation. In the following section, we propose a method to disentangle the effect of behavior policy and base the representation only on the task nature.

\section{Adversarial Data Augmentation}

% 3.1 Adversarial training objective
% 3.2 Adversarial policy learning
% 3.3 Overall algorithm

\begin{figure}
\centering
\includegraphics[width=\linewidth]{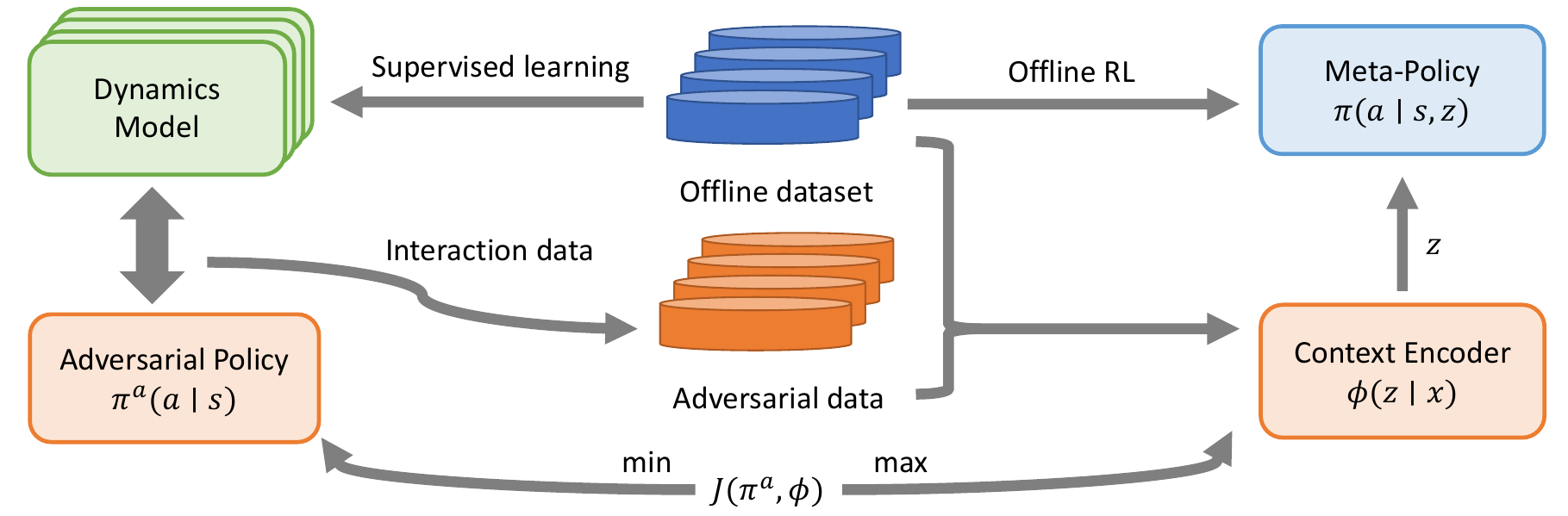}
\caption{The overall process of using adversarial data augmentation for offline meta-RL. }
\label{fig:framework}
\vspace{-2mm}
\end{figure}

In this section, we present a novel OMRL algorithm that can disentangle behavior policies during task representation learning via an effective approach called adversarial data augmentation, as shown in Figure~\ref{fig:framework}. This augmentation process, different from generating in-distribution auxiliary data, aims to provide a more robust data distribution that can minimize the influence of specific behavior policies during training. Therefore, the consequent task representation can break down the spurious relationship between behavior policies and task representations. In Section~\ref{subsec:adv}, we will introduce the adversarial training objective for our data augmentation phase. In Section~\ref{subsec:policy}, we demonstrate how we utilize a model-based RL approach to acquire additional data for task representation learning in the offline setting. In Section~\ref{subsec:framework}, we develop the algorithmic framework to use adversarial data argumentation in OMRL and present its implementation details.

\subsection{Adversarial Training Objective}
\label{subsec:adv}

As mentioned in Section~\ref{subsec:repr-learning}, we can derive a lower bound in the form of InfoNCE ~\citep{infonce} from the mutual information between task representations and the task MDP, which can be described as the following objective: 
\begin{equation}
    \label{eq:original-obj}
    \begin{aligned}
        J(\phi) = \mathbb{E}_{M \sim P(\cdot), z, z^* \sim \phi(\cdot \mid x_M)} \log \frac{\exp (z^\top z^*)}{\Phi(z)}, \\
        \text{where } \Phi(z) = \sum\nolimits_{M' \sim P(\cdot), z' \sim \phi(\cdot \mid x_{M'}) } \exp(z^\top z').
    \end{aligned}
\end{equation}
Here we specify the score function $S(\cdot, \cdot)$ in Equation~\eqref{eq:infonce} to be the inner product, which aligns to prevalent contrastive learning works~\citep{simcse, clip}. Though these works suggest that InfoNCE can be an effective objective to learn powerful representations according to given labels, the objective leads to building a connection between task representations $z$ and task context data $x_M$, rather than the task $M$ itself. Since the context data $x_M$ is sampled from the offline dataset, collected by specific behavior policies, this objective establishes a spurious relationship between behavior policies and learned task representations. In Section~\ref{subsec:toy} we will show an instantiation of this spurious relationship.

To eliminate such a spurious relationship, we envision that the context encoder can derive correct task representations from not only in-distribution data but also context data collected by any policy. To realize such generalization, the context encoder is expected to learn from a wider range of context data, which may help marginalize the effect of behavior policies. However, it is intractable to involve diverse context data collected by disparate policies in offline training. Inspired by previous progress in adversarial training~\citep{gan, gail}, we find that an intuitive approach is to make the context encoder capable of discerning the most indistinguishable samples. Therefore, the context encoder can possess adequate robustness to identify the task no matter what behavior policies the data comes from. 

To generate additional context data for adversarial training, we adopt a model-based RL approach~\citep{mbrl1, mbrl2} to learn environment models and thus provide interaction data in the offline setting. For a particular meta-RL task distribution $P(M)$, previous methods usually consider task variants on environment dynamics $T$ and reward functions $r$. As discussed in the introduction section, we only consider task variations where environment dynamics differ since it is the key factor in the spurious relationship.
% Nonetheless, we note that the issue of spurious relationship only occurs when the environment dynamics $T$ varies. The reason is that changing rewards of each interaction will not affect subsequent decisions when the behavior policy is fixed, while the variation of transition can result in distinctly different trajectories which induces distribution shift. 
Given a multi-task offline dataset $\{D_i\}_{i=1}^N$, we learn several dynamics models via supervised learning with transition tuples sampled from the dataset. These learned dynamics models formulate a new task distribution $\hat{P}(\hat{M})$ that is quite similar to the real task distribution. Denoting $x_{\hat{M}}$ as context data collected from the simulating task $\hat{M}$ with an auxiliary policy $\pi^a$, we summarize our adversarial optimization problem in the max-min form below:
\begin{equation} 
    \label{eq:adv-obj}
    \begin{aligned}
    \max_{\phi}\min_{\pi^a} J(\pi^a, \phi) =  \mathbb{E}_{\hat{M} \sim \hat{P}(\cdot), z, z^* \sim \phi(\cdot \mid x_{\hat{M}}(\pi^a))} \log \frac{\exp(z^\top z^{*})}{\Phi(z)}, \\
    \text{where } \Phi(z) = \sum\nolimits_{\hat{M}' \sim \hat{P}(\cdot), z' \sim \phi(\cdot \mid x_{\hat{M}'}(\pi^a)) } \exp(z^\top z').
    \end{aligned}
\end{equation}
Note that we mark the context data $x_{\hat{M}}(\pi^a)$ for its sources from the simulating MDP $\hat{M}$ and the auxiliary policy $\pi^a$. Here we name the auxiliary policy $\pi^a$ as \textit{adversarial policy} since it is trained against the context encoder's goal. Using the InfoNCE objective $J(\pi^a, \phi)$ as the basic form, the minimization part guarantees that the adversarial policy can learn to search the most indistinguishable data for a specific task. Then the maximization part helps the task identification of the context encoder in the worst-case data. Therefore, the context encoder is expected to recognize tasks well for any data distribution, as it has learned from the weakest data the adversarial policy chooses to generate.

\subsection{Adversarial Data Generation}
\label{subsec:policy}

Starting from Equation~\eqref{eq:adv-obj}, we consider how to generate data from a simulating task distribution. To compensate for the absence of the task MDP in offline settings, we first train several dynamics models for each task data $D_i$ like previous model-based RL approaches~\citep{mbrl1, mbrl2}, following a supervised learning manner: 
\begin{equation}
    \max\nolimits_{\hat{T}_{i}} \mathbb{E}_{(s,a,s')\sim D_i} [\log \hat{T}_{i}(s'|s,a)].
\end{equation}
In addition, the reward function is learned in a similar supervised learning way with all of the task data, as the reward functions are shared across different tasks. 
According to Equation~\eqref{eq:adv-obj}, we require context data from simulating task distribution, which is determined by both the simulating environment model and an adversarial policy. As the environment model should precisely reflect the original task distribution, the functionality of producing adversarial data to confound the context encoder, i.e., minimizing $J(\pi^a, \phi)$, falls to the adversarial policy $\pi^a$. Since context data forms a sequence-style format, we adopt reinforcement learning to optimize the adversarial policy, which helps make the goal of obscuring task identification consistent along the whole trajectory. Specifically, we define the reward as how much ambiguity the task representations are measured by InfoNCE compared to the previous step. Letting $R(z_t) = \mathbb{E}_{z^*} [\log \frac{\exp(z_t^\top z^*)}{\Phi(z_t)}]$ denote the value of $J(\pi^a, \phi)$ at step $t$, the reward for training $\pi^a$ towards the adversarial objective is:
 % defined as 
\begin{equation}
    \label{eq:adv-reward}
    r^{\text{adv}}_t = R(z_{t+1}) - R(z_t) = \mathbb{E}_{z^*} \left[\log \frac{\exp(z_{t+1}^\top z^*)}{\Phi(z_{t+1})} - \log \frac{\exp(z_t^\top z^*)}{\Phi(z_t)}\right].
\end{equation}
Intuitively, we reward the adversarial policy for taking actions that will decrease the task representation similarity within the same task. Therefore, the adversarial policy $\pi^a$ is helpful to mine the most indistinguishable data to enhance the robustness of the context encoder $\phi$.
Nevertheless, we notice that adopting the reward function of Equation~\eqref{eq:adv-reward} is not sufficient to train a well-behaved adversarial policy. The learned policy may explore states that are far beyond the environment model's support, which leads to useless data. On the other hand, when only following the adversarial reward, the policy will learn to make decisions that are irrelevant to the task goals. To this point, we highlight two additional terms that should be incorporated with the adversarial reward:

% TODO: specific uncertainty form
\begin{itemize}
    \item \textit{Uncertainty penalty}. Like previous offline model-based RL works~\citep{morel, mopo}, we quantify the uncertainty $u_{t}(s, a)$ of the current state-action pair by the aleatoric uncertainty of the learned dynamics models on the state transition. Specifically, the aleatoric uncertainty is the maximum standard deviation of the Gaussian distribution of the next states among a group of dynamics models.
    % disagreement of learned dynamics models on the state transition. Specifically, the disagreement is the standard deviation of next states predicted by a group of learned dynamics models. 
    This uncertainty metric serves as a penalty for rollouts in the models to prevent the policy from collecting imprecise data. 
    \item \textit{Task completeness}. We also append the task reward $\hat{r}(s, a)$ provided by the learned reward model to the total reward used for training the adversarial policy. Adopting the task reward should be a natural approach since we focus more on contexts that are related to the task goal rather than other meaningless data. 
\end{itemize}
We also draw theoretical inspirations for the design of these two reward components. We illustrate these reward terms are beneficial to generate useful data in Appendix~\ref{app:theorem}. Combining the two additional reward terms above, we can acquire the total reward to train the adversarial policy: 
\begin{equation}
\label{eq:adv}
    \begin{aligned}
    J(\pi^a) = \mathbb{E}_{\hat{M} \sim \hat{P}(\cdot), s_t \sim \hat{M}, a_t \sim \pi^a(\cdot \mid s_t)} \Big[& \sum_{t=0}^{\infty}\gamma^t \big( r^{\text{adv}}_t - \\
    & \lambda_1 u_t(s,a) + \lambda_2 \hat{r}(s,a) \big) \Big],
    \end{aligned}
\end{equation}
where we get states $s_t$ from the learned environment model and actions $a_t$ from the adversarial policy $\pi^a$ for each task. $\lambda_1$ and $\lambda_2$ are two coefficients to balance the weights of different reward terms.  We adopt the branch rollout approach \citep{mbpo} to mitigate the influence of model error and avoid the requirement of initial state distribution. Instead of using a single dataset, we draw initial states from the joint dataset $D = \bigcup_{i} D_i$ to enhance the coverage of rollout. 

As for the maximization part of Equation~\eqref{eq:adv-obj}, the context encoder $\phi$ can learn from the adversarial data generated by the adversarial policy in the environment model by fixing the adversarial policy. The learning process is through simple contrastive learning with the InfoNCE objective: 
\begin{equation}
    J(\phi) = \mathbb{E}_{\hat{M} \sim \hat{P}(\cdot), z, z^* \sim \phi \left(\cdot \mid x_{\hat{M}}(\pi^a) \right)} \log \frac{\exp(z^\top z^{*})}{\Phi(z)}.
\end{equation}

\begin{algorithm}[t]
\caption{Meta-training with adversarial data augmentation}
\label{alg:reda}
\textbf{Input}: The offline dataset $\{D_i\}_{i=1}^n$ from $n$ tasks, the context encoder $\phi(z| x)$, the meta-policy $\pi(a| s, z)$, and the adversarial policy $\pi^a(a| s)$
\begin{algorithmic}[1]
    \STATE $\triangleright$ \textit{Dynamics model training}
    \FOR{$i = 1, 2, \cdots, n$}
        \STATE Learn $m$ dynamics models with supervised learning on $D_i$ to represent the task dynamics
    \ENDFOR

    \STATE $\triangleright$ \textit{Task representation learning}
    \FOR{$K= 1, \dots, K_c$}
        \STATE $\triangleright$ \textit{Collect adversarial data}
        \FOR{$i = 1, 2, \cdots, n$}
            \STATE Sample adversarial data $\hat{D}_i$ from learned dynamics models $\hat{M}_i$ using adversarial policy $\pi^a$
        \ENDFOR

        \STATE $\triangleright$ \textit{Train the adversarial policy}
        \STATE Compute rewards from sampled adversarial data $\hat{D}_1, \dots, \hat{D}_n$ according to Equation~\eqref{eq:adv}
        \STATE Update adversarial policy $\pi^a(a| s)$ using SAC

        \STATE $\triangleright$ \textit{Train the context encoder}
        \STATE Update context encoder $\phi(z | x)$ with sampled adversarial data $\hat{D}_1, \dots, \hat{D}_n$ 
        
    \ENDFOR
    
    \STATE $\triangleright$ \textit{Meta-policy learning}
    \STATE Update meta-policy $\pi(a| s, z)$ with dataset $\{D_i\}_{i=1}^n$ with fixed context encoder $\phi$

\end{algorithmic}
\end{algorithm}

\subsection{Overall Framework}
\label{subsec:framework}

% Like common generative adversarial learning approaches, we train the adversarial policy and the context encoder 

Though most of our work focuses on the stage of task representation learning, to evaluate its effectiveness, we present the overall OMRL algorithm that meta-trains a policy using adversarial data augmentation. As shown in Algorithm~\ref{alg:reda}, the meta-training process consists of three stages: task dynamics model training, task representation learning, and meta-policy learning. Utilizing provided offline datasets, we first train several dynamics models for each task that are prepared for data generation. During task representation learning, we train the adversarial policy $\pi^a$ and the context encoder $\phi$ alternatively like common generative adversarial learning approaches \citep{gan}. Finally, we fix the context encoder $\phi$ and train the meta-policy with a standard offline RL algorithm. To be more specific, we adopt the classic SAC algorithm~\citep{sac} and add a standard behavior-cloning-style regularization term~\citep{td3bc} for its offline adaptation. The context encoder $\phi(z\mid x)$ has a Transformer-based structure~\citep{attention} for capturing sequential inputs. For details in computing self-attention modules, the query vector comes from the current observation and the last action, and the context data in the form of state-action pairs serves as sources of key and value vectors. For policy optimization, we detach the gradient of output task representation and feed the representation into the meta-policy $\pi(a\mid s, z)$ along with the observed state $s$. The network structure of the meta-policy is a simple multi-layer perceptron (MLP), which is also the same as the adversarial policy $\pi^a (a\mid s)$ despite the absent input of task representation $z$. 

\section{Related Work}

\subsection{Offline Reinforcement Learning}

Offline reinforcement learning proposes to learn policies from a fixed offline dataset and is considered helpful for many realistic domains such as healthcare and robotics~\citep{DBLP:conf/icra/JohanninkBNLKLO19, gottesman2019guidelines} to avoid exhaustive online interactions that is usually impractical for real-world scenarios. The main issue with offline learning is the distribution shift caused by the mismatch between the behavior and exploitation policies~\citep{levine_offlinerl_survey}, inducing a huge challenge for online evaluation. 
When adopting RL algorithms from online RL literature, offline model-free RL methods usually adopt a conservative way to learn policies. For example, they explicitly constrain the policy to be similar to the behavior policy with different regularization manners~\citep{bcq, bear, awac, td3bc, awac, f2c, dae, wpc, bpr, hfgb}, imposing conservative penalty on out-of-distribution actions~\citep{cql, edac, pbrl}, or re-weighting offline samples during policy evaluation and improvement~\citep{iql, iac}. 

Another line of research is offline model-based RL that incorporates dynamics models~\citep{mopo, morel, combo, rambo, maple, mobile, jin2022hybrid} to synthesize data for better generalization. However, the dynamics models are typically trained from offline dataset via supervised learning and still face model errors on unseen state-action pairs. One approach to account for this problem is to modify the learned dynamics~\citep{morel} or add reward penalties to highly uncertain areas~\citep{mopo}, which is quantified by disagreement or maximum standard deviation over an ensemble of dynamics. COMBO~\citep{combo} bypasses uncertainty quantification by leveraging the model-free method for optimization. Recently, RAMBO~\citep{rambo} introduces the idea of robust RL~ and proposes to learn adversarial dynamics models to incorporate conservatism. 

Our paper generally adopts offline model-free RL approaches to train meta-policies by introducing SAC~\citep{sac} with adding a behavior-cloning-style regularization term~\citep{td3bc}. When tackling the auxiliary adversarial policy, we also utilize a model-based RL approach. We train several dynamics models to augment the original dataset and penalize the policy by the model uncertainty, which builds a connection with the model-based offline RL literature~\citep{morel, mopo}. 

\subsection{Offline Meta-Reinforcement Learning}

Meta-reinforcement learning~\cite{metarl-survey} is a popular research area that learns a policy that is capable of adapting to any new task from the task distribution with as little data as possible. When more and more meta-RL approaches are shown to learn a good meta-policy with extremely comprehensive interactions~\citep{maml, mgrl}, offline meta-RL~\cite{omrl_survey} extends the boundary of meta-RL by learning the meta-policy directly from offline data. 

Different from a major research direction on applying model-agnostic meta-learning (MAML) methods~\cite{maml} to offline meta-RL~\citep{macaw}, we here concentrate more on adopting context-based meta-RL approaches~\citep{pearl}. 
As MAML-style approaches often require few-shot updates for task adaptation, context-based meta-RL methods can utilize pre-collected context data from other policies or its bootstrapped interactions, realizing a promising capability of zero-shot generalization. 
When extending context-based meta-RL methods to the offline setting, previous works typically employ contrastive-style learning objectives to learn discriminative representations for each task, such as triplet loss~\citep{mbml}, InfoNCE~\citep{infonce, corro}, and a negative-power variant contrastive loss~\cite{focal}. Similar to our motivation, MBML~\citep{mbml} and CORRO~\citep{corro} also point out the impact of behavior policy on task identification, and try to alleviate this issue via relabeling data with known reward functions or learned generative models like conditional variational autoencoders. Nonetheless, we note that these works do not indeed eliminate the correlation between the behavior policy and the task transition. CSRO~\citep{ibomrl} aims to solve this problem by introducing an information bottleneck between the policy and task features. However, we notice that previous works assume that the behavior policies used to collect the dataset are diversified enough (e.g. CORRO uses each saved checkpoint during the training process to collect data), which alleviates the issue of building the spurious relationship with behavior policies. To the best of our knowledge, we are the first to exhaustively consider the spurious relationship in task representation learning and eliminate the effect of behavior policies.

\begin{figure} 
\centering
\includegraphics[width=1\linewidth]{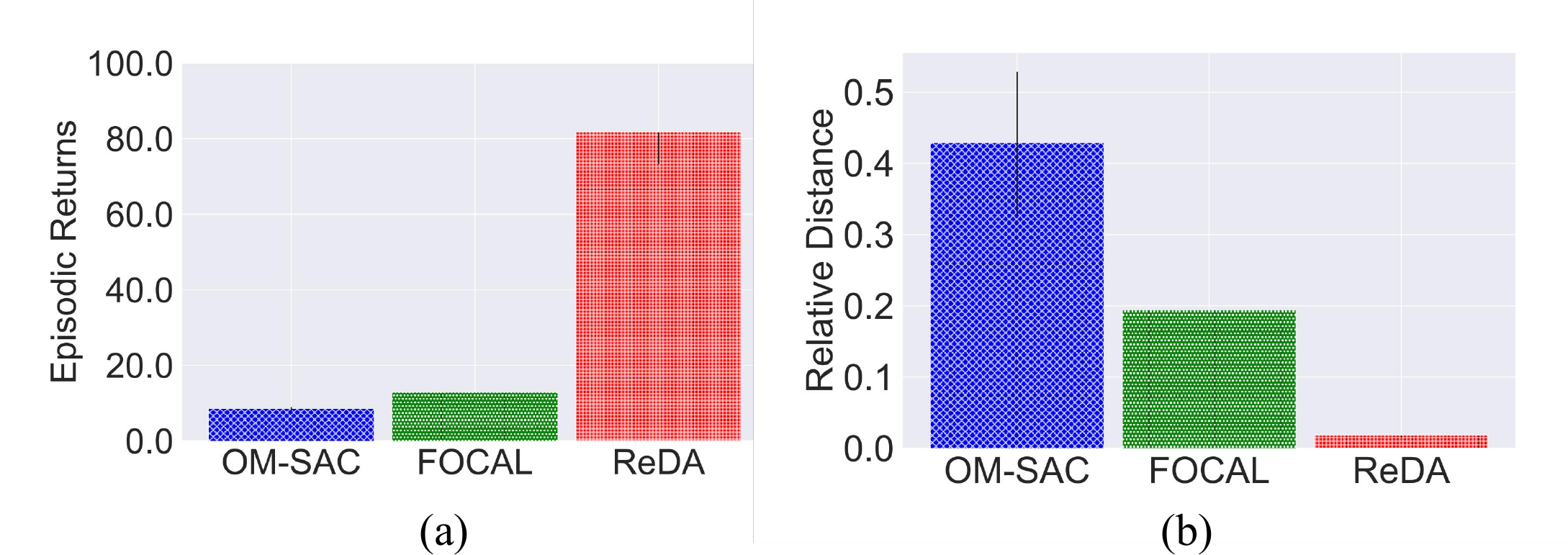}
\vspace{-6mm}
\caption{ (a) Performance on \texttt{InvertedPendulum-v2} with 1.0x gravity coefficient. (b) The relative representation metric (from Equation~\eqref{eq:relative-metric}) of different methods.}
\label{fig:toy}
\vspace{-4mm}
\end{figure}

% on-policy
\begin{table*}[t]
\caption{Average scores of the meta-policy under the on-policy protocol. The performance is averaged over $5$ random seeds. We abbreviate task names for simplicity. For example, \texttt{Walker2d Gravity-1} denotes the \texttt{Walker2d-v2} task with training datasets collected from $1$ checkpoint.}
\vskip 0.05in
\label{table:on}
\centering
\begin{tabular}{l|rrrrrr}
\toprule
Task Type & \multicolumn{1}{c}{OM-BC} & \multicolumn{1}{c}{OM-SAC} & \multicolumn{1}{c}{CORRO} & \multicolumn{1}{c}{FOCAL} & \multicolumn{1}{c}{PromptDT} & \multicolumn{1}{c}{ReDA (ours)}  \\
\midrule

Walker2d Gravity-1 & 835.7 $\pm$ 158.9 & 1227.4 $\pm$ 176.8 & 965.7 $\pm$ 283.5 & 1260.2 $\pm$ 292.1 & 1019.7 $\pm$ 216.4 & \textbf{1720.8 $\pm$ 321.3} \\
Walker2d Gravity-3 & 1110.4 $\pm$ 170.6 & 1180.1 $\pm$ 175.7 & 1269.8 $\pm$ 309.9 & 1584.6 $\pm$ 361.4 & 1125.0 $\pm$ 226.0 & \textbf{1611.6 $\pm$ 288.0} \\
Walker2d Gravity-5 & 657.4 $\pm$ 265.8 & 1216.5 $\pm$ 388.4 & 1412.5 $\pm$ 387.7 & 2181.7 $\pm$ 355.7 & 1082.8 $\pm$ 277.7 & \textbf{2560.7 $\pm$ 362.1} \\
Walker2d Dof-Damping-1 & 2692.3 $\pm$ 488.5 & 2995.4 $\pm$ 287.9 & 2756.5 $\pm$ 379.5 & 3409.2 $\pm$ 343.9 & 2635.0 $\pm$ 298.6 & \textbf{3769.2 $\pm$ 364.2} \\
Walker2d Dof-Damping-3 & 1734.5 $\pm$ 309.0 & 1941.4 $\pm$ 342.5 & 2232.0 $\pm$ 332.4 & 1940.7 $\pm$ 321.9 & 1935.1 $\pm$ 342.3 & \textbf{3536.7 $\pm$ 258.7} \\
Walker2d Dof-Damping-5 & 998.8 $\pm$ 231.3 & 1478.9 $\pm$ 267.6 & 1647.5 $\pm$ 423.3 & 2222.0 $\pm$ 311.0 & 1136.0 $\pm$ 281.2 & \textbf{2497.6 $\pm$ 339.4} \\
Ant Gravity-1 & 3155.4 $\pm$ 347.4 & 2943.5 $\pm$ 533.6 & 3154.4 $\pm$ 698.3 & 3420.9 $\pm$ 536.6 & 3137.1 $\pm$ 437.9 & \textbf{3536.8 $\pm$ 479.8} \\
Ant Gravity-3 & 3673.5 $\pm$ 471.3 & 3623.4 $\pm$ 495.4 & 3845.0 $\pm$ 502.3 & 3763.1 $\pm$ 452.9 & 3665.7 $\pm$ 437.9 & \textbf{4403.4 $\pm$ 464.8} \\
Ant Gravity-5 & 3273.5 $\pm$ 417.4 & 3543.7 $\pm$ 576.0 & 3451.1 $\pm$ 549.6 & 3681.3 $\pm$ 567.2 & 3114.1 $\pm$ 476.4 & \textbf{3735.2 $\pm$ 498.3} \\
Ant Dof-Damping-1 & 1849.5 $\pm$ 378.3 & 3243.6 $\pm$ 395.4 & 3107.5 $\pm$ 299.0 & 2387.7 $\pm$ 380.5 & 2764.7 $\pm$ 374.1 & \textbf{3409.2 $\pm$ 332.1} \\
Ant Dof-Damping-3 & 3119.1 $\pm$ 382.9 & 3385.6 $\pm$ 395.3 & 3289.4 $\pm$ 375.4 & 3160.2 $\pm$ 361.2 & 3378.4 $\pm$ 358.5 & \textbf{3804.2 $\pm$ 412.3} \\
Ant Dof-Damping-5 & 2391.9 $\pm$ 312.3 & 2422.7 $\pm$ 379.4 & 2243.7 $\pm$ 432.3 & 2461.6 $\pm$ 356.9 & 2145.9 $\pm$ 309.2 & \textbf{2547.7 $\pm$ 366.7} \\
Hopper Gravity-1 & 1002.0 $\pm$ 266.6 & 1453.6 $\pm$ 564.4 & 903.4 $\pm$ 134.4 & 1148.6 $\pm$ 344.4 & 1123.9 $\pm$ 333.0 & \textbf{2069.6 $\pm$ 436.5} \\
Hopper Gravity-3 & 1403.0 $\pm$ 405.5 & 1534.7 $\pm$ 398.6 & 700.5 $\pm$ 121.9 & 1562.5 $\pm$ 348.3 & 1264.2 $\pm$ 355.2 & \textbf{1784.0 $\pm$ 415.8} \\
Hopper Gravity-5 & 1086.7 $\pm$ 184.8 & 1495.5 $\pm$ 317.9 & 622.7 $\pm$ 200.8 & 594.3 $\pm$ 443.6 & 1001.6 $\pm$ 232.7 & \textbf{1539.1 $\pm$ 493.1} \\
Hopper Dof-Damping-1 & 1228.3 $\pm$ 233.3 & 1104.6 $\pm$ 258.3 & 787.3 $\pm$ 132.5 & 835.0 $\pm$ 309.5 & 971.9 $\pm$ 265.5 & \textbf{1419.6 $\pm$ 296.7} \\
Hopper Dof-Damping-3 & 1463.4 $\pm$ 390.8 & 1455.6 $\pm$ 388.1 & 768.2 $\pm$ 191.1 & 1445.6 $\pm$ 356.3 & 1223.7 $\pm$ 377.4 & \textbf{1573.0 $\pm$ 400.5} \\
Hopper Dof-Damping-5 & 1047.3 $\pm$ 291.4 & 1461.7 $\pm$ 371.7 & 892.8 $\pm$ 164.3 & 1356.0 $\pm$ 239.2 & 1133.9 $\pm$ 186.5 & \textbf{1552.3 $\pm$ 326.5} \\
HalfCheetah Gravity-1 & 4204.0 $\pm$ 840.3 & 4641.7 $\pm$ 745.6 & 5353.7 $\pm$ 869.6 & 3651.9 $\pm$ 505.9 & 5034.1 $\pm$ 644.3 & \textbf{6772.0 $\pm$ 649.5} \\
HalfCheetah Gravity-3 & 3719.9 $\pm$ 614.6 & 5419.6 $\pm$ 682.0 & 5054.6 $\pm$ 754.7 & 2950.0 $\pm$ 391.7 & 3900.9 $\pm$ 944.0 & \textbf{6130.8 $\pm$ 856.4} \\
HalfCheetah Gravity-5 & 2684.8 $\pm$ 584.5 & 4398.2 $\pm$ 778.3 & 4165.7 $\pm$ 654.7 & 1863.9 $\pm$ 633.4 & 3494.8 $\pm$ 711.6 & \textbf{4549.0 $\pm$ 753.9} \\

\midrule
Average &2063.4&2484.2&2315.4&2232.4&2204.3&\textbf{3072.5}  \\

\bottomrule
\end{tabular}
\end{table*}

% off-policy
\begin{table*}[t]
\caption{Average scores of the meta-policy under the off-policy protocol. The performance is averaged over $5$ random seeds.} %We abbreviate task names for simplicity. For example, \texttt{Walker2d Gravity-1} denotes the \texttt{Walker2d-v2} task with training datasets collected from $1$ checkpoint.
\vskip 0.05in
\label{table:off}
\centering
\begin{tabular}{l|rrrrrr}
\toprule
Task Type & \multicolumn{1}{c}{OM-BC} & \multicolumn{1}{c}{OM-SAC} & \multicolumn{1}{c}{CORRO} & \multicolumn{1}{c}{FOCAL} & \multicolumn{1}{c}{PromptDT} & \multicolumn{1}{c}{ReDA (ours)}  \\
\midrule

Walker2d Gravity-1 & 102.9 $\pm$ \ \ 16.5 & 773.9 $\pm$ 115.4 & 354.9 $\pm$ 113.3 & 972.4 $\pm$ 314.9 & 516.5 $\pm$ 165.3 & \textbf{1349.8 $\pm$ 374.2} \\
Walker2d Gravity-3 & 887.5 $\pm$ 197.7 & 1096.0 $\pm$ 182.6 & 755.7 $\pm$ 387.5 & 1331.1 $\pm$ 384.5 & 996.6 $\pm$ 204.1 & \textbf{1336.3 $\pm$ 363.3} \\
Walker2d Gravity-5 & 655.4 $\pm$ 253.2 & 983.0 $\pm$ 394.0 & 899.1 $\pm$ 298.5 & 2318.6 $\pm$ 312.3 & 1254.8 $\pm$ 255.0 & \textbf{2421.0 $\pm$ 398.7} \\
Walker2d Dof-Damping-1 & 2566.4 $\pm$ 358.6 & 2567.3 $\pm$ 381.1 & 2138.2 $\pm$ 432.6 & 2141.7 $\pm$ 299.3 & 2594.3 $\pm$ 315.6 & \textbf{3274.5 $\pm$ 429.5} \\
Walker2d Dof-Damping-3 & 1599.2 $\pm$ 340.2 & 1611.1 $\pm$ 415.6 & 1765.8 $\pm$ 498.6 & 1619.0 $\pm$ 362.0 & 1883.9 $\pm$ 331.2 & \textbf{3236.0 $\pm$ 322.7} \\
Walker2d Dof-Damping-5 & 823.6 $\pm$ 340.1 & 1798.6 $\pm$ 345.6 & 1667.2 $\pm$ 438.8 & 1943.4 $\pm$ 353.4 & 1446.5 $\pm$ 326.5 & \textbf{2243.7 $\pm$ 352.2} \\
Ant Gravity-1 & 2512.5 $\pm$ 295.1 & 2673.2 $\pm$ 464.3 & 1876.9 $\pm$ 527.6 & 3240.3 $\pm$ 346.8 & 2895.3 $\pm$ 349.5 & \textbf{3392.5 $\pm$ 343.9} \\
Ant Gravity-3 & 2910.7 $\pm$ 283.4 & 2982.0 $\pm$ 422.7 & 1575.8 $\pm$ 675.5 & 3442.5 $\pm$ 337.2 & 3243.7 $\pm$ 322.4 & \textbf{4128.6 $\pm$ 348.3} \\
Ant Gravity-5 & 3064.7 $\pm$ 385.9 & 2764.7 $\pm$ 445.5 & 1647.8 $\pm$ 624.2 & 3437.4 $\pm$ 357.1 & 3049.1 $\pm$ 364.1 & \textbf{3529.1 $\pm$ 337.4} \\
Ant Dof-Damping-1 & 2319.5 $\pm$ 373.3 & 3125.4 $\pm$ 362.8 & 2469.4 $\pm$ 446.6 & 2171.1 $\pm$ 363.3 & 2638.5 $\pm$ 332.6 & \textbf{3367.4 $\pm$ 302.7} \\
Ant Dof-Damping-3 & 2580.4 $\pm$ 291.1 & 2848.9 $\pm$ 399.0 & 2507.5 $\pm$ 355.7 & 3022.2 $\pm$ 318.5 & 2723.8 $\pm$ 317.7 & \textbf{3727.2 $\pm$ 356.4} \\
Ant Dof-Damping-5 & 2380.7 $\pm$ 228.1 & 2218.4 $\pm$ 360.1 & 2174.6 $\pm$ 289.0 & 2591.0 $\pm$ 338.7 & 2441.0 $\pm$ 206.5 & \textbf{2611.3 $\pm$ 382.5} \\
Hopper Gravity-1 & 1394.2 $\pm$ 330.2 & 1322.9 $\pm$ 455.2 & 925.6 $\pm$ 107.1 & 1328.8 $\pm$ 306.7 & 1305.5 $\pm$ 346.4 & \textbf{1674.5 $\pm$ 512.6} \\
Hopper Gravity-3 & 1497.8 $\pm$ 406.2 & 1429.3 $\pm$ 348.3 & 671.3 $\pm$ 223.9 & \textbf{1844.8  $\pm$ 741.2} & 1574.3 $\pm$ 366.7 & 1521.2 $\pm$ \ 376.0 \\
Hopper Gravity-5 & 1262.6 $\pm$ 412.6 & 1318.6 $\pm$ 331.5 & 754.8 $\pm$ 208.7 & 750.1 $\pm$ 299.9 & 1396.6 $\pm$ 276.3 & \textbf{1489.6 $\pm$ 319.9} \\
Hopper Dof-Damping-1 & 1241.7 $\pm$ 144.3 & 991.5 $\pm$ 189.2 & 804.3 $\pm$ 122.8 & 905.2 $\pm$ 345.9 & 1112.9 $\pm$ 164.6 & \textbf{1255.7 $\pm$ 306.4} \\
Hopper Dof-Damping-3 & 1161.4 $\pm$ 279.9 & 1136.9 $\pm$ 372.9 & 863.4 $\pm$ 212.7 & 672.8 $\pm$ 301.7 & 1284.7 $\pm$ 244.3 & \textbf{1440.0 $\pm$ 369.6} \\
Hopper Dof-Damping-5 & 715.8 $\pm$ 220.2 & 1387.3 $\pm$ 416.6 & 637.6 $\pm$ 275.7 & 722.9 $\pm$ 295.3 & 910.4 $\pm$ 264.2 & \textbf{1492.2 $\pm$ 334.3} \\
HalfCheetah Gravity-1 & 6024.0 $\pm$ 768.2 & 3985.1 $\pm$ 676.9 & 5244.7 $\pm$ 766.6 & 3960.4 $\pm$ 517.0 & 5755.3 $\pm$ 582.9 & \textbf{6243.8 $\pm$ 604.4} \\
HalfCheetah Gravity-3 & 2112.9 $\pm$ 601.1 & 4792.5 $\pm$ 596.7 & 4865.2 $\pm$ 710.8 & 2293.1 $\pm$ 532.1 & 3211.5 $\pm$ 660.6 & \textbf{5749.2 $\pm$ 599.0} \\
HalfCheetah Gravity-5 & 3006.2 $\pm$ 684.3 & 4164.7 $\pm$ 602.3 & 3654.6 $\pm$ 677.0 & 2181.2 $\pm$ 414.6 & 3243.8 $\pm$ 601.1 & \textbf{4268.1 $\pm$ 562.2} \\

\midrule

Average &1943.8&2189.1&1821.6&2042.4&2165.7&\textbf{2845.3} \\

\bottomrule
\end{tabular}
\end{table*}

\section{Experiments}

In this section, we conduct experiments to evaluate the proposed adversarial data augmentation\footnote{Code is available at \url{https://github.com/LAMDA-RL/ReDA}}. We investigate whether the task \textbf{Re}presentation learning via adversarial \textbf{D}ata \textbf{A}ugmentation (ReDA) can disentangle the effect of behavior policies and thus helps train generalizable meta-policies.
In Section~\ref{subsec:toy}, we use a didactic example to show the issue of the spurious relationship and how ReDA can help address it via a classic control task. In Section~\ref{subsec:pb}, we establish tasks and datasets to benchmark ReDA and other OMRL baselines. We particularly design two evaluation manners called on-policy and off-policy protocols to verify the effectiveness of learned task representation under different circumstances. Finally in Section~\ref{subsec:qtf}, we ablate multiple design choices in the adversarial data augmentation process, for example, the use of models and different reward terms on training adversarial policy. We also provide illustrative explanations on task representations in out-of-distribution data.

\begin{figure*} 
\centering
\includegraphics[width=0.9\textwidth]{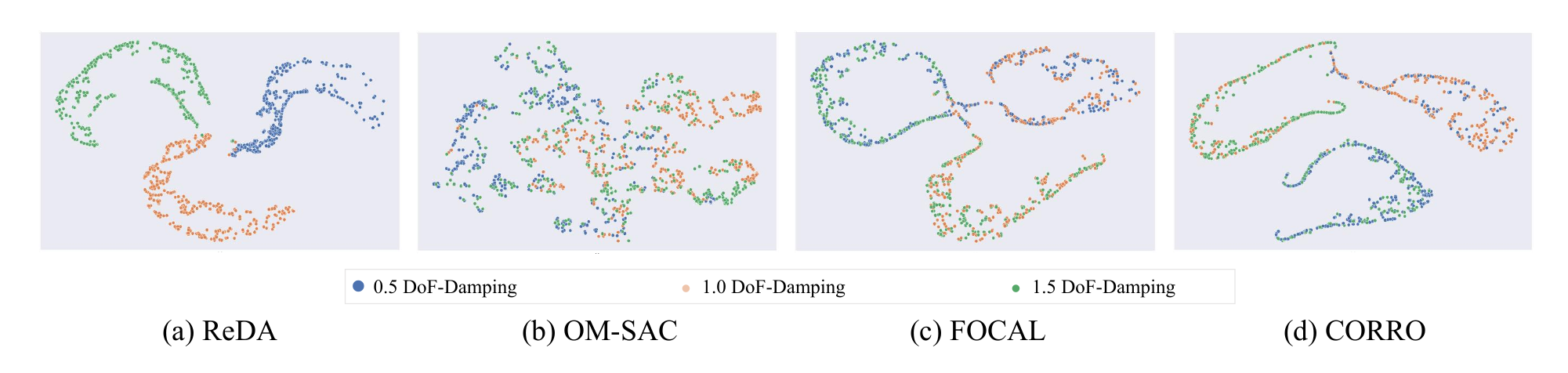}
\vspace{-2mm}
\caption{Visualization on task representations with t-SNE dimensionality reduction for (a) ReDA, (b) OM-SAC, (c) FOCAL, and (d) CORRO on the task set Walker2d Dof-Damping-1. Task representations from different tasks are shown in distinct colors.}
\label{fig:visual}
\vspace{-3mm}

\end{figure*}

\subsection{A Didactic Example}
\label{subsec:toy}

As a straightforward example to show the spurious relationship between behavior policies and task representations, we conduct experiments on the \texttt{InvertedPendulum-v2} environment from Gym's classic control benchmarks. To make task variants, we modify the gravity of the environment by multiplying it by $1.0$ and $2.0$, creating two distinct tasks.
For offline dataset collection, we independently train SAC policies on each task. We select the checkpoint that achieves an expected episodic return of $100$ for gravity $1.0$ (task 1) and that of $50$ of gravity $2.0$ (task 2) to collect training data. We denote these two datasets as $D_1^\text{train}$ and $D_2^\text{train}$. After training, we choose the test dataset from task 1 with returns of $50$, which is denoted by $D_1^\text{test}$. To highlight the spurious relationship issue, we evaluate the meta-policy on task 1 with context data from $D_1^\text{test}$. We speculate that the existence of the spurious relationship may encourage the context encoder to predict task representations more similar to task 2. We test the performance with several baselines. FOCAL is an OMRL baseline~\citep{focal} that designs an ad-hoc contrastive learning objective for task representation learning. 
OM-SAC is a baseline that adopts the same context encoder structure as ReDA but only adopts the same offline RL manner as ReDA without data augmentation and context encoder learning.

We depict the performance in task 1 when using context data from $D_1^\text{test}$ in Figure \ref{fig:toy}(a). The majority of baselines exhibit poor performance, indicating a failure in task identification and generalization. In contrast, our method ReDA can better capture task features and thus performs exceptionally well when facing out-of-distribution contexts.
To get deeper insights on how well the task representations are, we define a relative representation metric $d(\phi)$ to delineate whether the task representation is close to task 1 given its out-of-distribution context data, as follows: 
\begin{equation}
    \label{eq:relative-metric}
    d(\phi) = \frac{|D_2^\text{train}|\sum_{z_1 \sim \phi(\cdot | x_1), x_1 \sim D_1^\text{train}} \sum_{z \sim \phi(\cdot | x), x \sim  D_1^\text{test}} \|z_1 - z\|^2}{|D_1^\text{train}|\sum_{z_2 \sim \phi(\cdot | x_2), x_2 \sim D_2^\text{train}} \sum_{z \sim \phi(\cdot | x), x \sim D_1^\text{test}} \|z_2 - z\|^2}.
\end{equation}
We acquire this metric by feeding context data from each training dataset and computing the mean square errors with task representations from test data. It is obvious that when $d(\phi)$ is close to $0$, the context encoder correctly recognizes most context data as task 1. A large value of $d(\phi)$ means the context encoder cannot perform task identification well. As illustrated in Figure \ref{fig:toy}(b), we find that only the context encoder of ReDA reaches a $d(\phi)$ near $0$, indicating its effectiveness. In contrast, OM-SAC and FOCAL both have a larger $d(\phi)$ compared to ReDA, while the particular contrastive objective of FOCAL makes it slightly better than OM-SAC. We include a detailed description of this example in Appendix~\ref{app:details}.

% $$d(\phi) = \frac{|D_2^\text{train}|\sum_{z_1 \sim \phi(\cdot | x_1), x_1 \sim D_1^\text{train}} \sum_{z \sim \phi(\cdot | x), x \sim  D_1^\text{test}} \|z_1 - z\|^2}{|D_1^\text{train}|\sum_{z_2 \sim \phi(\cdot | x_2), x_2 \sim D_2^\text{train}} \sum_{z \sim \phi(\cdot | x), x \sim D_1^\text{test}} \|z_2 - z\|^2}.$$

\subsection{Performance on MuJoCo Benchmarks}
\label{subsec:pb}

Following the experiment setups in previous meta-RL works~\cite {corro, focal, luo}, we construct multiple tasks by varying certain hyper-parameters of the simulator in MuJoCo locomotion tasks~\citep{mujoco}. Specifically, we choose a set of Gym's MuJoCo environments, including \texttt{HalfCheetah-v2}, \texttt{Hopper-v2}, \texttt{Walker2d-v2}, and \texttt{Ant-v2}. For each environment, we perturb the environment parameters, namely \texttt{gravity} and \texttt{dof-damping}, of the original simulator to create a wide range of tasks. During offline training, the meta-policy is trained when the environment parameters are set to $0.5$, $1.0$, and $1.5$ times of their original values for each environment. We also choose $0.8$ and $1.2$ times of the parameters as unseen tasks to evaluate the generalization ability. 
For each task, we train behavior policies of multiple qualities via SAC~\citep{sac}. Similar to that in CORRO~\citep{corro}, we save the checkpoint during the whole learning process but only select $1$, $3$, and $5$ checkpoints from all to collect training data and use the rest checkpoints to collect test data. Our data collection is distinguished from CORRO, which utilizes all checkpoints to collect training data. We name the task type by the number of checkpoints used for training. For example, \texttt{Walker2d Gravity-1} means that the offline dataset of Walker2d is collected by $1$ checkpoint in 0.5x, 1.0x, and 1.5x gravity. Besides the above baselines, we also include OM-BC, which trains a meta-policy with behavior cloning, and two additional OMRL baselines CORRO~\cite{corro} and PromptDT~\cite{pdt}. To better show the generalization ability of different OMRL algorithms, we apply two evaluation protocols to shape different context data: \textit{on-policy protocol} and \textit{off-policy protocol}. 

\textbf{On-policy protocol. }In the on-policy protocol, the context data comes from interactions the meta-policy encounters when deployed to the task. Compared to the behavior policies that collect the offline data, the meta-policy is often an improved one, and thus there still exists a mismatch between training-time and test-time distribution. We report the performance in Table~\ref{table:on}. The results demonstrate that ReDA significantly outperforms other baseline methods when using bootstrapped context data, indicating the superior ability of task recognition and generalization. Previous works like FOCAL and CORRO exhibit similar performance to OM-SAC, indicating that their task representation learning approaches cannot work for this extreme setting of poor data coverage.

\textbf{Off-policy protocol. }To further investigate the generalization ability of the context encoders, we introduce the off-policy protocol where we replace the context with interactions collected by unseen checkpoints. This evaluation type typically raises a higher requirement for the generalization capability of the context encoder since it should output correct task representations collected from arbitrary policies. As shown in Table \ref{table:off}, we notably find that almost all methods have performance drops in the off-policy protocol, while our algorithm still achieves significant performance advantages compared to other baselines.

\begin{table}[t]
	\caption{Ablations on different variants of ReDA. The variant w/o model has no data augmentation and the variant w/o adv. augments data with rollouts from random policies. \texttt{DofDamp} is short for \texttt{Dof-Damping}.}
	\label{tab:abld}
	\begin{tabular}{lrrr}\toprule
		Task Type & w/o model & w/o adv. & ReDA \\ \midrule
        Walker2d Gravity-1 & 834.0 & 1197.5 & \textbf{1349.8 $\pm$ 321.3} \\
		Walker2d Gravity-3 & 1135.5 & 1268.9 & \textbf{1336.3 $\pm$ 363.3} \\
        Walker2d Gravity-5 & 1348.6 & 1220.2 & \textbf{2421.0 $\pm$ 398.7} \\
		Walker2d DofDamp-1 & 2606.5 & 3186.2 & \textbf{3274.5 $\pm$ 429.5} \\
		Walker2d DofDamp-3 & 1773.3 & 2845.8 & \textbf{3236.0 $\pm$ 322.7} \\
        Walker2d DofDamp-5 & 1641.7 & 1566.0 & \textbf{2243.7 $\pm$ 352.2} \\
  
  \bottomrule
	\end{tabular}
\end{table}

% abl rew
\begin{table}[t]
    \caption{Ablations on adversarial policy learning with different reward terms. The variant w/o UP learns the policy is ReDA with $\lambda_1=0$. The variant w/o TC is ReDA with $\lambda_2=0$. \texttt{DofDamp} is short for \texttt{Dof-Damping}.}
	% \caption{Ablations on adversarial policy learning with different reward terms. The variant w/o UP learns the policy without uncertainty penalty ($\lambda_1=0$). The variant w/o TC learns the policy without task completeness reward ($\lambda_2=0$). All results are averaged over $5$ seeds. ``DofDamp'' is short for Dof-Damping.}
	\label{tab:ablr}
	\begin{tabular}{lrrr}\toprule
		Task Type & w/o UP & w/o TC & ReDA \\ \midrule
        Walker2d Gravity-1 & 1185.4 & 1281.3 & \textbf{1349.8 $\pm$ 321.3} \\
		Walker2d Gravity-3 & 1249.0 & 1292.3 & \textbf{1336.3 $\pm$ 363.3} \\
        Walker2d Gravity-5 & 1572.1 & 2212.5 & \textbf{2421.0 $\pm$ 398.7} \\
		Walker2d DofDamp-1 & 2893.2 & 3002.4 & \textbf{3274.5 $\pm$ 429.5} \\
		Walker2d DofDamp-3 & 2985.2 & 3145.8 & \textbf{3236.0 $\pm$ 322.7} \\
        Walker2d DofDamp-5 & 1787.8 & 2209.2 & \textbf{2243.7 $\pm$ 352.2} \\
  
  \bottomrule
	\end{tabular}
\end{table}

\subsection{Ablation Studies}
\label{subsec:qtf}
We design a few ablation studies to investigate why ReDA works in the experiments above. To compare how well the method identifies different tasks, we visualize the task representations trained in the \texttt{Walker2d Dof-Damping-1} setting in Figure~\ref{fig:visual}. The task representations are obtained by feeding the context encoders with interaction data from tasks including different gravity coefficients. We find that ReDA can generate generally distinct task representations for different tasks even though some data patterns are never encountered during training, indicating that the proposed adversarial data augmentation helps improve out-of-distribution generalization. When we remove the adversarial data augmentation of ReDA and keep other procedures the same (shown in the results of OM-SAC), the consequent task representations become indistinguishable, as well as those from FOCAL and CORRO. 
% Although CORRO claims to realize robust task representation learning through a bi-level encoder structure, it seems that CORRO cannot identify tasks effectively within the poor-coverage data. 
To ablate each component in our adversarial data augmentation process, we also design two variants of ReDA and compare their performance on \texttt{Walker2d Gravity} and \texttt{Walker2d Dof-Damping} dataset. One variant is to implement ReDA but does not introduce pre-trained dynamics models for data augmentation (w/o model). The other variant is to utilize the model for data augmentation but with a random policy rather than an additional adversarial policy (w/o adv.). As shown in Table~\ref{tab:abld}, the results on the Walker2d task sets demonstrate that both two variants degrade severely compared to ReDA, which verifies the effectiveness of the proposed adversarial data augmentation. Besides, The performance of the ReDA variant without data augmentation is not as high as the ReDA variant using dynamics models, indicating that introducing additional context data from pre-trained dynamics models is helpful for generalization of task representations. 
Then we turn to ablate the design of reward for learning the adversarial policy. As we adopt two additional reward terms in our methodology, we propose two variants of ReDA. One variant (w/o UP) is ReDA without uncertainty penalty ($\lambda_1 = 0$) and the other (w/o TC) is ReDA without task completeness reward ($\lambda_2 = 0$). The results in Table~\ref{tab:ablr} show that our reward design could promote task representation generalization while these two variants show less promising performance. The uncertainty penalty, which is adopted from offline model-based literature~\citep{mopo}, is shown to be more critical to policy performance. We also provide more results of task representation visualization on mentioned ablation variants in Appendix~\ref{app:details}, which shows that ReDA outputs more distinct representations than those variants.

\section{Conclusion}
\label{sec:conclu}

In this paper, we shift our attention to OMRL in a low-data regime, where the offline dataset is limited in terms of coverage by collective policy. We underscore the criticality of disentangling behavior policies from task representation learning and offer a practical solution named adversarial data augmentation. We train an adversarial policy to prevent the context encoder from identifying the task and augment the offline dataset. The data is collected in learned dynamics models, which is irrelevant to the original behavior policies and thus can assist in learning robust task representations. Experimental results reveal that when the dataset is encompassed by limited behavior policies, previous OMRL methods encounter difficulties with task identification. Nevertheless, task representation learning with adversarial data augmentation resolves this problem, effectively generalizing to out-of-distribution context data at test time. Our aspiration for this research is to stimulate further exploration into the impact of behavior policies and the standardization of OMRL benchmarks. One identified limitation of our work is the rudimentary designation of the context encoder, with us merely inheriting a simple Transformer architecture, tasked with extracting and aggregating context information. Due to its departure from this study's primary focus, we reserve this issue for future work.

%%%%%%%%%%%%%%%%%%%%%%%%%%%%%%%%%%%%%%%%%%%%%%%%%%%%%%%%%%%%%%%%%%%%%%%%

%%% The acknowledgments section is defined using the "acks" environment
%%% (rather than an unnumbered section). The use of this environment 
%%% ensures the proper identification of the section in the article 
%%% metadata as well as the consistent spelling of the heading.

\begin{acks}
This work is supported by National Science Foundation of China (61921006) and Tencent AI Lab (RBFR2023011). We thank the anonymous reviewers for their support and helpful discussions. 
\end{acks}

%%%%%%%%%%%%%%%%%%%%%%%%%%%%%%%%%%%%%%%%%%%%%%%%%%%%%%%%%%%%%%%%%%%%%%%%

%%% The next two lines define, first, the bibliography style to be 
%%% applied, and, second, the bibliography file to be used.
\bibliographystyle{ACM-Reference-Format} 
\balance
\bibliography{sample}

%%%%%%%%%%%%%%%%%%%%%%%%%%%%%%%%%%%%%%%%%%%%%%%%%%%%%%%%%%%%%%%%%%%%%%%%

\onecolumn
\appendix

\section{Theoretical Explanation}
\label{app:theorem}
In this section, we analyze the quality of data provided by learned models and adversarial policy. Since the learned models may be inaccurate, we investigate the trustworthiness of the data. For the group of ensemble MDP models as the surrogate dynamics $\hat{M}_i=(\hat{T}_i, \hat{r})$ for each original task, we assume the model error is bounded by $\epsilon_m$: $\max_{i} \max_{s,a} [D_{\rm TV}(T_{i}(\cdot|s,a), \hat{T}_{i}(\cdot|s,a))] \leq \epsilon_m$, where $T_i$ is the ground truth transition of $i$-th task. We use the maximum $TV$ divergence of transitions as the distance between two models. Then we define the learning process by a definition of mapping, which we called $\epsilon$-mapping:

\begin{definition}{($\epsilon$-mapping).} \label{def:map}
$f_{\epsilon}: \mathcal{M} \mapsto \mathcal{M}$ is an $\epsilon$-mapping, if $\forall M \in \mathcal{M}$ and $\hat{M} = f_{\epsilon}(M) \in \mathcal{M}$, where $T$ and $\hat{T}$ are their transitions correspondingly, $D_{\rm TV}(T(\cdot|s,a), \hat{T}(\cdot|s,a)) \leq \epsilon$ holds.
\end{definition}

This definition implies that if the model learning process could guarantee the model error less than $\epsilon_m$ and the learned model could still be a valid task to guarantee $f_{\epsilon}(M) \in \mathcal{M}$, which is the support of the whole task distribution, then the learning process is an $\epsilon_m$-mapping $f_{\epsilon_m}$. We will use $f$ as a substitute for $f_{\epsilon_m}$ in the following of our paper. 

The learned dynamics models construct a surrogate task distribution $\hat{P}(M)$. Correspondingly we define the probability space of the surrogate task distribution: $(\mathcal{M}, 2^{\mathcal{M}}, \hat{P})$, where $2^{\mathcal{M}}$ is a power set of $\mathcal{M}$ to serve as the event field, and we consider that the surrogate distribution is mapped from the raw probability space $(\mathcal{M}, 2^{\mathcal{M}}, P)$ based on a mapping $f$, which satisfies $\hat{P}({\mathcal{M}_{s}}) = P(f^{-1}(\mathcal{M}_{s})) $, $\forall \mathcal{M}_{s} \subset \mathcal{M}$. 
To measure the distance between distribution $P$ and its surrogate $\hat{P}$, we make an $L$-Lipschitz assumption for the probability measure $P$:

\begin{assumption}{(Task Distribution Smoothness).} 
\label{assump:taskds}
Given the probability space $(\mathcal{M}, 2^\mathcal{M}, P)$ and any subset $\mathcal{M}_{s}$ of $\mathcal{M}$, the probability measure $P$ satisfies that
$$
|P(\mathcal{M}_s) - P(f(\mathcal{M}_{s}))| \leq L \epsilon.
$$
\end{assumption}

The assumption reveals that if two tasks are similar in the task distribution, their probabilities are close as well, especially in continuous control tasks where the hyper-parameters of environments are often continuous. Under the condition of smoothness of the task distribution, we obtain the following Theorem:

\begin{theorem}{(Task Distribution Shift).} \label{theorem:taskg}
Given the probability space of meta-task distribution $(\mathcal{M}, 2^{\mathcal{M}}, P)$, for each model $M$ sampled according to $P$, we learned surrogate model $\hat{M}$ and let $\hat{P}$ be the task distribution associated with the real model $\hat{M}$, the TV divergence between $P$ and $\hat{P}$ can be bounded by:
$$
D_{\rm TV}(P(M), \hat{P}(M)) \leq L\epsilon_m.
$$

\end{theorem}

\begin{proof}

\begin{equation*}
    \begin{split}
        D_{\rm TV}(P(M), \hat{P}(M)) = \sup_{\mathcal{M}_s \subset \mathcal{M}} [P(\mathcal{M}_s) - P(f(\mathcal{M}_s))] \leq L\epsilon_m.
    \end{split}
\end{equation*}

\end{proof}

This theorem reveals the potential model bias when learning task distribution can be well reduced. The core idea of controlling model bias is that we should guarantee an $\epsilon$-mapping. To satisfy this, the model error of the collected data should be low so that the learned model could still be within the support of the whole task distribution. To guarantee the data rollout with low model errors, we thus consider the uncertainty penalty. To guarantee the constructed learned model is a valid surrogate task, we perform task-relevant data searching, which suits the design of task reward.

\section{Details of Experiments}
\label{app:details}

\subsection{MuJoCo Tasks}

% We select some locomotion environments in MuJoCo tasks, such as \texttt{Walker2d-v2}, \texttt{Hopper-v2}, \texttt{Ant-v2} and \texttt{HalfCheetah-v2}. Each environment is controlled by some physical parameters such as gravity, dof-damping and so on. We choose the parameters that could influence the transitions. For example, as shown in Figure~\ref{fig:meta}, with the same action sequence and initial state, the task of $dof-damping=1.0$ could obtain a stable walking trajectory, while that of $0.5$ will fall down forward and $1.5$ will fall down backward, which needs distinguished decision in that three tasks. That implies that the task identification is important.

To evaluate our methods, we have selected several locomotion environments from the MuJoCo tasks, including \texttt{Walker2d-v2}, \texttt{Hopper-v2}, \texttt{Ant-v2}, and \texttt{HalfCheetah-v2}. Here is a brief introduction to these four tasks:
 
\textbf{Hopper-v2}. The Hopper environment in MuJoCo simulates a 2D one-legged robot, resembling a grasshopper or a pogo stick. The objective of the task is to control the robot's actions in order to make it hop and maintain balance while moving forward. The agent needs to learn how to apply the correct forces to the leg joints to generate hopping motions and effectively navigate the environment. The goal is to develop control policies that enable the robot to achieve stable and efficient hopping movements.

\textbf{Walker2d-v2}. The Walker2d environment in MuJoCo represents a 2D bipedal robot with two legs. The goal of the task is to control the robot's movements and enable it to walk while maintaining stability. The agent must learn to generate coordinated leg movements and adjust joint angles to achieve smooth and balanced walking gaits. The objective is to develop control policies that allow the robot to navigate the environment with stable and efficient walking motions.

\textbf{HalfCheetah-v2}. The HalfCheetah environment in MuJoCo simulates a 2D model of a cheetah-like quadruped robot. The objective of the task is to control the robot's actions and enable it to achieve fast and efficient running. The agent needs to learn how to generate coordinated leg movements and apply appropriate forces to the leg joints in order to maximize speed and maintain stability during locomotion. The goal is to develop control policies that allow the robot to exhibit agile and stable running behaviors.

\textbf{Ant-v2}. The Ant environment simulates a four-legged ant-like robot in a 3D environment. The objective of the task is to control the movements of the ant and navigate it through the environment. The agent needs to learn to coordinate the leg movements and apply appropriate forces to the leg joints to achieve efficient locomotion. The ant robot must adapt to different terrains, avoid obstacles, and maintain balance while moving toward a specified target or goal. The agent's goal is to learn effective control policies that enable the ant to navigate and accomplish tasks in the environment.

Each environment is controlled by various physical parameters such as \texttt{gravity}, \texttt{dof-damping}, and others. We specifically choose parameters that can significantly impact the transitions. For instance, as depicted in Figure~\ref{fig:meta}, when considering the same action sequence and initial state, the task with $dof$-damping set to $1.0$ results in a stable walking trajectory, while the task with $0.5$ causes a forward fall, and one with $1.5$ leads to a backward fall. These distinct outcomes imply that accurate task identification is significant to OMRL. We mainly create task variants based on different parameters on \texttt{gravity} and \texttt{dof-damping}. However, we found that \texttt{dof-damping} has little influence on the \texttt{HalfCheetah-v2} environment. Therefore, we only consider variants of \texttt{gravity} for \texttt{HalfCheetah-v2} tasks.

\subsection{Task Set Construction}
% \label{app:tasks}

\begin{figure} 
\centering
\includegraphics[width=0.8\linewidth]{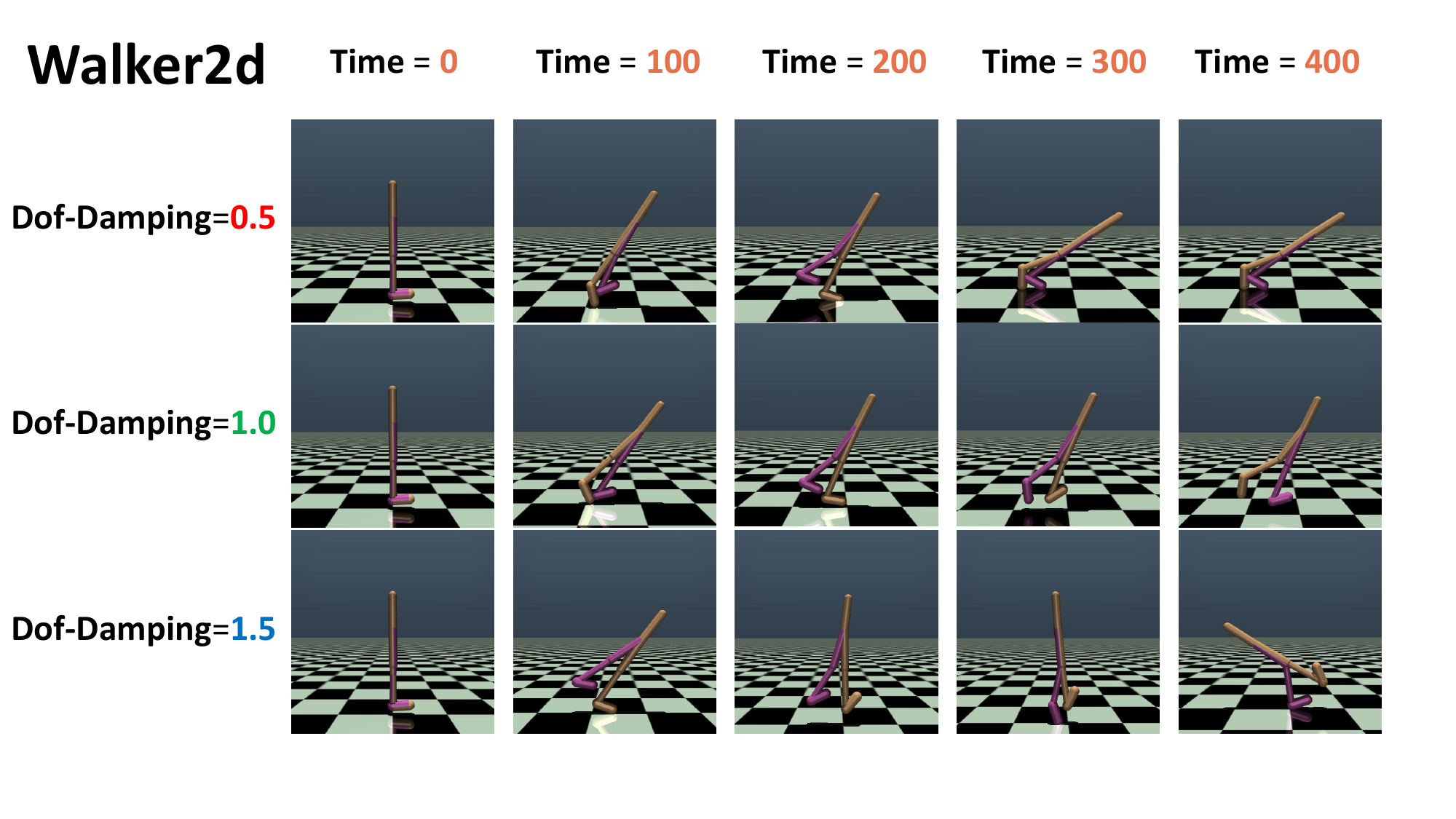}
\vspace{-3mm}
\caption{The visual trajectories of Walker2d-v2 with distinct hyper-parameters of \texttt{dof-damping}, including $0.5$, $1.0$, and $1.5$. Each trajectory is fed with the same initial state and following action sequence.}
\label{fig:meta}
\vspace{-5mm}
\end{figure}

To collect the offline dataset, we utilized saved checkpoints of the policy during the learning process of the SAC algorithm. Specifically, we trained SAC for $1$ million time steps and saved checkpoints at every interval of $20,000$ time steps, resulting in a total number of $50$ checkpoints.
From these $50$ checkpoints, we selected a subset as our training data. For tasks with only $1$ checkpoint, we chose the $25$-th checkpoint as the training data.
For tasks with $3$ checkpoints, we selected the $5$-th, $25$-th, and $45$-th checkpoints as our training data.
In the case of tasks with $5$ checkpoints, we selected the $5$-th, $10$-th, $25$-th, $40$-th, and $45$-th checkpoints as our training data. For all cases, the remained checkpoints were used for the test dataset.

To ensure a fixed data size for different task sets, we maintained a total data size of $1$ million transitions across all tasks. For instance, for the task set having $5$ checkpoints, each checkpoint contributes $200,000$ transitions to the dataset. This approach allowed us to maintain a universal dataset size, regardless of the number of checkpoints used.

\subsection{Details of Didactic Examples}
\begin{figure} 
\centering
\includegraphics[width=0.9\linewidth]{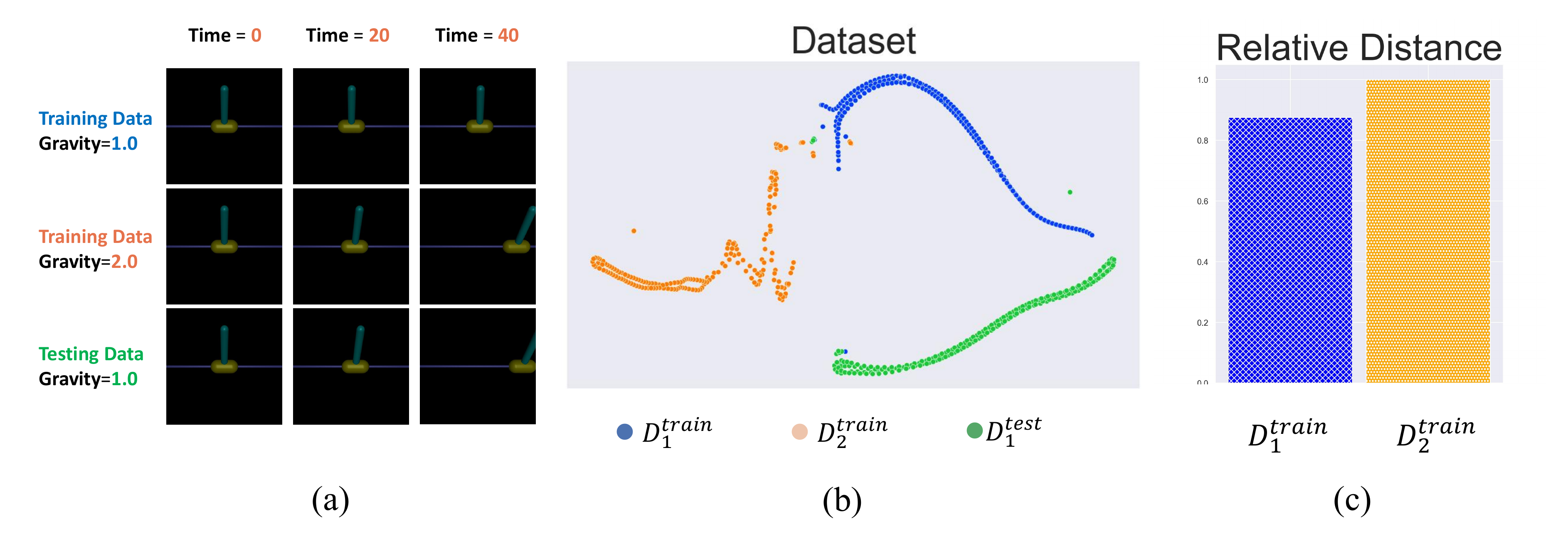}
\vspace{-3mm}
\caption{Visualization on the training and test datasets in \texttt{InvertedPendulum-v2} from our didactic example. (a) The rendered images from the environment in different datasets. (b) The t-SNE plot of state-action pairs in different datasets. (c) The relative representation metrics (from Equation\eqref{eq:relative-metric}) of two training datasets compared to the test dataset.}
\label{fig:toy_rela}
\vspace{-5mm}
\end{figure}

We present additional information about the data in our didactic example from Figure~\ref{fig:toy_rela}. In Figure~\ref{fig:toy_rela} (a), we observe that the training dataset $D_{2}^{\text{train}}$ and the test dataset $D_{1}^{\text{test}}$ have similar trajectories because of similar policy performance, though they are collected from two distinct tasks. This similarity indicated the potential existence of the spurious relationship since people are also difficult to identify these two tasks. 
In Figure~\ref{fig:toy_rela} (b), we utilize the dimensionality reduction method to visualize the inner state-action pairs of each dataset. The results show that the test dataset may not be close to the dataset $D_1^{\text{train}}$ within the same task. 
To further delineate the data properties, we compute the relative representation metrics of each training dataset compared with the test dataset, as mentioned in Equation~\eqref{eq:relative-metric} from the main paper. The results in Figure~\ref{fig:toy_rela} (c) confirm that the distance within the same task tends to be smaller, showing that it is feasible for task representations to identify different tasks.
% tasks with shared features exhibit closer proximity.

k

\subsection{Baseline Algorithms}
\label{app:baselines}

We present introductions to the baselines from recent works, including FOCAL \cite{focal} and CORRO \cite{corro} to further illustrate their and our contributions. 

\textbf{FOCAL}. FOCAL is a pioneering algorithm that introduces the direction of offline meta-reinforcement learning. Its primary objective is to develop a metric that effectively distinguishes the offline dataset across different tasks. FOCAL employs a contrastive learning objective to learn effective representation from established positive and negative data.

\textbf{CORRO}. CORRO aims to train a robust context encoder using a bi-level structure. It incorporates noise as negative label data into the low-level encoder to enhance its resilience. Then CORRO trained robust encoders through contrastive learning. CORRO utilizes a transformer-based structure as the high-level context encoder to extract task features. However, the inclusion of noisy negative data, though enhancing robustness, restricts the context encoder within the confines of the data distribution. This solution may degrade generalization and exacerbate spurious relationships between task features and behavior policies.

% We show details of our baselines in the paper, including the learning process and their code sources for implementation. 
For FOCAL and CORRO, we used the implementation from their codebase in \url{https://github.com/LanqingLi1993/FOCAL-ICLR} and \url{https://github.com/PKU-RL/CORRO}. For the task representation learning process, all hyper-parameters are kept the same as the original implementations. Our modification is to replace the dataset with our collected data. However, we found that their original policy learning processes may not be unified, leading to unfair comparisons. Therefore, we unify the meta-policy learning process with the same SAC+BC algorithm as ReDA in experiments.

\begin{figure*} 
\centering
\includegraphics[width=1\linewidth]{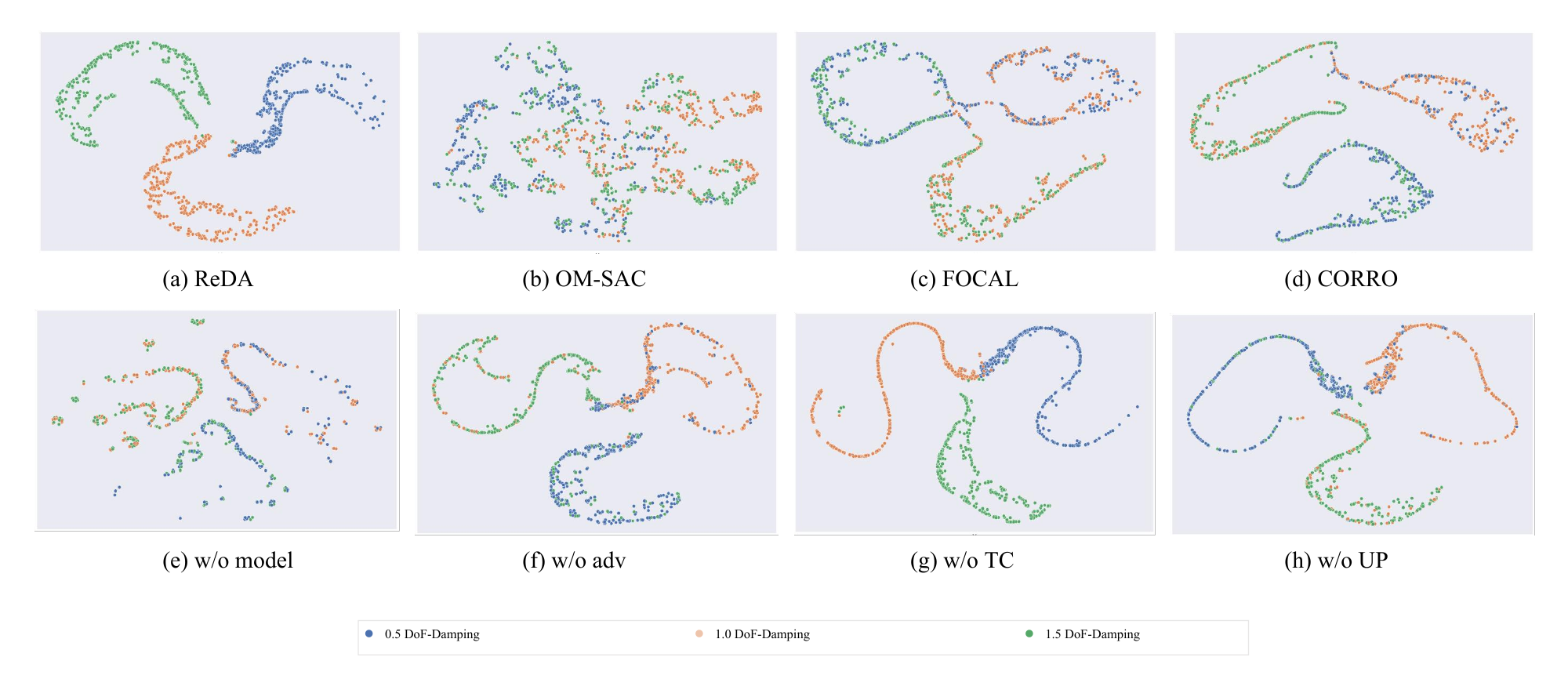}
\vspace{-3mm}
\caption{Visualization with t-SNE dimensionality reduction on task representations of OM-SAC, CORRO, FOCAL, ReDA w/o model, w/o adv, w/o TC, and w/o UP in the task set \texttt{Walker2d-Dof-Damping-1}. Different tasks have distinct colors.}
\label{fig:total}
\vspace{-2mm}
\end{figure*}

\subsection{Implementation Details of ReDA}
% \label{app:param}

\textbf{Model Learning.  }We adopt a neural network with $3$ hidden layers and a hidden dimension of $200$ for each layer. The learning scheme is similar to MOPO. Batch size, learning rate, and ensemble size are $2048$, $1e-3$ and $3$, respectively. As there are $3$ training tasks for each run, the total number of dynamics models we trained is $9$. All the modules are optimized with Adam optimizer, with the learning rate and batch size equal to $3\times 10^{-4}$ and $256$. 

\noindent \textbf{Context Encoder and Policy Learning.  }The context encoder is a Transformer network with an embedding dimension of $64$, one layer, and one attention head. The output representations are then fed into a simple linear layer to produce the task feature with the dimension of $8$. For policy learning, the actor and critic are both parameterized by neural networks with two hideen layers of $256$  dimensions.

\noindent \textbf{Training Procedure.  } For all tasks, the total number of training epochs is set to $1000$. Within each epoch, we perform $1000$ steps of gradient updates. We set the hyper-parameter $\lambda_1$ and $\lambda_2$ to $1.0$ for all experiments. When performing branch rollout on dynamics models, the batch size is set to $2000$ for each model to generate data. 

\subsection{Computation Cost}
% \label{app:compt}

We conduct our experiments on $4$ workstations. Each of those is equipped with 32-core CPUs and 2 NVIDIA 2080Ti GPUs. The training time for each task is about $20$-$24$ hours.

\begin{figure*} 
\centering
\includegraphics[width=1\linewidth]{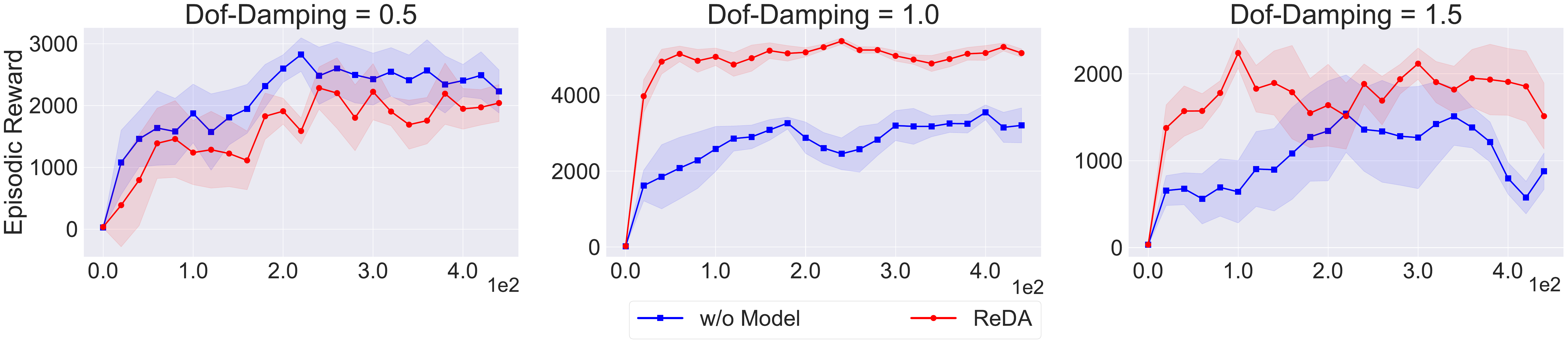}
\vspace{-3mm}
\caption{Average episodic rewards of ReDA and ReDA w/o model on the each training (seen) task from \texttt{Walker2d-Dof-Damping}. The X-axis is the number of epochs and the Y-axis is the episodic reward. Results are averaged over $5$ seeds.}
\label{fig:seen}
\vspace{-2mm}
\end{figure*}

\begin{figure*} 
\centering
\includegraphics[width=0.7\linewidth]{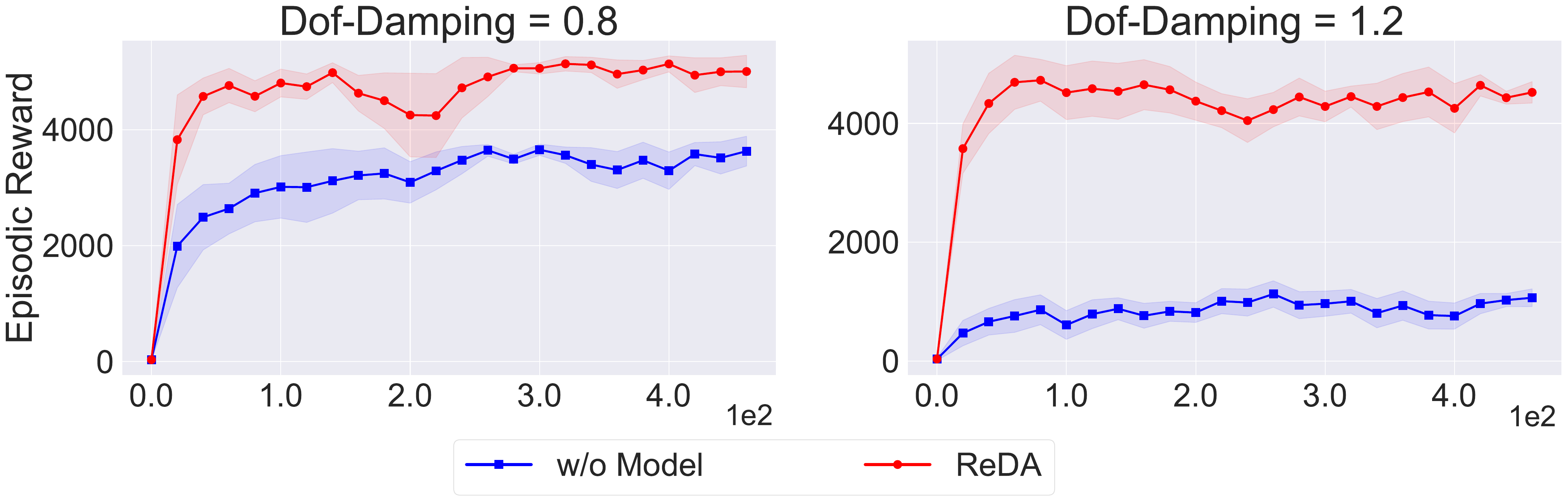}
\vspace{-3mm}
\caption{Average episodic rewards of ReDA and ReDA w/o model on each test (unseen) task from \texttt{Walker2d-Dof-Damping}. The X-axis is the number of epochs and the Y-axis is the episodic reward. Results are averaged over $5$ seeds.}
\label{fig:unseen}
\vspace{-2mm}
\end{figure*}

\subsection{Additional Visualization Results}

We also provide more visual results of the representation of the baselines and ablations in the task of Walker2d Dof-Damping $1$ in Figure~\ref{fig:total}. In the task of Walker2d Dof-Damping $1$, we conducted experiments comparing different baselines and ablations. We visually analyzed the results to gain further insights.
We observed that the baselines without model data augmentation (including baselines, ReDA w/o model) exhibited significantly poorer generalization performance on the test dataset compared to those with model data augmentation. This finding suggests that incorporating data augmentation techniques improved the generalization ability of the models.
However, we also noticed that when using model data augmentation with random augmentation techniques, there were instances where the introduced data had a detrimental effect on the generalization of the context encoder. This indicates that careful consideration and selection of data augmentation strategies are crucial to ensure beneficial effects on the overall performance.
Regarding the reward design, we found that even without incorporating uncertainty, the context encoder still struggled to handle certain testing scenarios. This suggests that uncertainty-aware reward design is essential for robust performance. Additionally, we observed that the absence of the task reward led to some corner cases where the model performance was compromised. This highlights the importance of incorporating task-specific rewards to guide the learning process effectively.
In summary, our visual analysis of the baselines and ablations in the Walker2d Dof-Damping $1$ task revealed that data augmentation, appropriate reward design, and task-specific rewards are all crucial factors in improving the generalization and performance of the context encoder models.

% We find that for the baselines without data augmentation (baselines, w/o model), the generalization of test dataset is much worse than that with model data augmentation. And for the task representation with model data augmentation, random augmentation may introduce some data that is harmful for the context encoder generalization. And for the reward design, without uncertainty, the context encoder could still not handle some testing and without the task reward, there exist some corner case, which implies that the reward design is necessary.

\subsection{Additional Results of Ablation Studies}

The t-SNE results demonstrate that the data generated from the learned model significantly improves task representation learning. To further validate the effectiveness of the learned model, we present curves for each task on the \texttt{Walker2d-Dof-Damping-1} task set and compare our results with the ReDA variant w/o model. In Figure~\ref{fig:seen}, we present the performance curves of three training tasks. The ReDA variant w/o model, which lacks the task identification ability as shown in the visualization, struggles to recognize different tasks but may only treat them as a hybrid task. Consequently, it exhibits conservative performance only in specific tasks from the training set. For example, it achieves relatively good performance in the task with a damping value of $0.5$ while failing to perform well on the other tasks. In contrast, ReDA performs well on all training tasks.
Moving on to the unseen tasks depicted in Figure~\ref{fig:unseen}, ReDA significantly outperforms the variant w/o model. This improvement can be ascribed to ReDA's capability to recognize and adapt to different tasks effectively. ReDA w/o model still has a less promising performance in unseen tasks.

% Figure~\ref{fig:seen} shows the curves on training tasks, including Dof-Damping with $0.5$, $1.0$ and $1.5$. Since ReDA w/o model could not distinguish the task well, it may treat the learned tasks as a whole MDP, which may lead to a conservative or robustness on a certain task among the training set. Thus, we find that it achieves a comparable good performance on the task of $0.5$ but fail in other tasks. While ReDA could behave well in most of the training tasks. As for the unseen tasks shown in Figure~\ref{fig:unseen}, ReDA outperforms w/o model greatly since it has no ability to identify the tasks.

\begin{table}
\caption{Ablations on different numbers of dynamics models. All results are averaged over $5$ seeds. }
\label{tab:abldetails}
\begin{tabular}{lrrr}\toprule
Task Type & 1 model & 3 models & 5 models \\ \midrule
Walker2d Gravity-1 (on-policy protocol) & 1466.8 & 1720.8 & 1743.1 \\
Walker2d Gravity-1 (off-policy protocol) & 1254.5 & 1349.8 & 1376.5 \\
\bottomrule
\end{tabular}
\end{table}

\subsection{Additional Results on Dynamics Models}

We also investigate the effects of ensemble size of the learned dynamics models in our method. We typically train 3 dynamics models in our experiments. To evaluate its sensitivity when the number changes, we also conduct experiments on \texttt{Walker2d Gravity-1} with 1/3/5 dynamics models. The average scores of both the on-policy protocol and the off-policy protocol are shown in Table~\ref{tab:abldetails}. We find our selection of 3 models can be a balance between performance and computation cost. 

%%%%%%%%%%%%%%%%%%%%%%%%%%%%%%%%%%%%%%%%%%%%%%%%%%%%%%%%%%%%%%%%%%%%%%%%

\end{document}